\documentclass{article} 
\usepackage[accepted]{icml2024}
\usepackage{amssymb,amsfonts}
\usepackage{algorithm}
\usepackage{algorithmic}
\usepackage{bbold}
\usepackage{hyperref}
\usepackage{xcolor}
\usepackage{enumitem}
\usepackage{pifont}
\newcommand{\cmark}{\ding{51}}%
\newcommand{\xmark}{\ding{55}}%
\newcommand{\ignore}[1]{}
\hypersetup{
    colorlinks,
    linkcolor={blue!50!black},
    citecolor={blue!50!black},
    urlcolor={blue!80!black}
}
\usepackage{amsmath,amsthm}
\DeclareMathOperator*{\argminA}{arg\,min}    
\usepackage{xcolor}
\definecolor{lightcyan}{rgb}{0.88, 1.0, 1.0}
\definecolor{lightyellow}{rgb}{1.0, 1.0, 0.88}
\definecolor{linen}{rgb}{0.98, 0.94, 0.9}
\definecolor{lightblue}{rgb}{0.87, 0.922, 0.968}
\definecolor{lightgray}{rgb}{0.95, 0.95, 0.95}
\usepackage{mdframed} 
\mdfdefinestyle{theoremstyle}{
  linecolor=blue,
  linewidth=0pt,
  backgroundcolor=linen,
}
\mdfdefinestyle{remarkstyle}{
  linecolor=blue,
  linewidth=0pt,
  backgroundcolor=lightblue,
}
\mdfdefinestyle{defstyle}{
  linecolor=blue,
  linewidth=0pt,
  backgroundcolor=lightgray,
}

\newmdtheoremenv[style=theoremstyle]{thm}{Theorem}
\newmdtheoremenv[style=theoremstyle]{lemma}{Lemma}
\newmdtheoremenv[style=defstyle]{assump}{Assumption}
\newmdtheoremenv[style=remarkstyle]{remark}{Remark}
\newmdtheoremenv[style=remarkstyle]{corollary}{Corollary}
\newmdtheoremenv[style=defstyle]{definition}{Definition}

\renewcommand{\citet}[1]{\citep{#1}}
\renewcommand{\cite}[1]{\citep{#1}}
\icmltitlerunning{Global Optimality of Actor-Critic Algorithms}

\begin{document}

\twocolumn[
\icmltitle{Closing the Gap: Achieving Global Convergence (Last Iterate) of Actor-Critic under Markovian Sampling with Neural Network Parametrization}
\icmlsetsymbol{equal}{*}

\begin{icmlauthorlist}
\icmlauthor{Mudit Gaur}{comp}
\icmlauthor{Amrit Singh Bedi}{comp2}
\icmlauthor{Di Wang}{comp3}
\icmlauthor{Vaneet Aggarwal}{comp4}
\end{icmlauthorlist}

\icmlaffiliation{comp}{Department of Statistics, Purdue University, West Lafayette, IN, U.S.A}
\icmlaffiliation{comp2}{Department of Computer Science, University of Central Florida}
\icmlaffiliation{comp3}{Department of Computer Science, KAUST}
\icmlaffiliation{comp4}{School of IE and School of ECE, Purdue University, West Lafayette, IN, U.S.A}

\icmlcorrespondingauthor{Mudit Gaur}{mgaur@purdue.edu}
\icmlcorrespondingauthor{Vaneet Aggarwal}{vaneet@purdue.edu}
\icmlkeywords{Machine Learning, ICML}

\vskip 0.3in
]
\printAffiliationsAndNotice{}
\begin{abstract} \label{abstract}
%
The current state-of-the-art theoretical analysis of Actor-Critic (AC) algorithms significantly lags in addressing the practical aspects of AC implementations. This crucial gap needs bridging to bring the analysis in line with practical implementations of AC. To address this, we advocate for considering the MMCLG criteria: \textbf{M}ulti-layer neural network parametrization for actor/critic, \textbf{M}arkovian sampling, \textbf{C}ontinuous state-action spaces, the performance of the \textbf{L}ast iterate, and \textbf{G}lobal optimality. These aspects are practically significant and have been largely overlooked in existing theoretical analyses of AC algorithms. 
In this work, we address these gaps by providing the first comprehensive theoretical analysis of AC algorithms that encompasses all five crucial practical aspects (covers MMCLG criteria). We establish global convergence sample complexity bounds of $\tilde{\mathcal{O}}\left({\epsilon^{-3}}\right)$. We achieve this result through our novel use of the weak gradient domination property of MDP's and our unique analysis of the error in critic estimation.

\end{abstract}

\section{Introduction} \label{introduction}

Actor-Critic (AC) algorithms \citep{konda1999actor} have emerged as a cornerstone in modern reinforcement learning, showcasing remarkable versatility and effectiveness across a diverse range of applications such as games \citep{zhou2022pac}, network scheduling \cite{agarwal2022multi}, robotics \citep{chen2022scalable}, autonomous driving \citep{tang2022highway}, and video streaming \citep{zhang2022duasvs}. At their core,  AC algorithms aim to maximize the expected returns, denoted by $J(\lambda)$, where $\lambda \in \mathbb{R}^{d}$ is the policy parameter and $J(\lambda)$ is the expected reward for a policy parameterised by $\lambda$. The algorithms involves an interplay between the gradient ascent for the actor parameter and the estimation of the action value function or critic. Theoretical analysis of AC algorithms, however, often lags behind their practical implementations. Due to the practical use of multi-layer neural network parameterizations for the actor as well as critic, an important observation in recent research is the widening gap between theoretical models and real-world applicability. True insights are gained when theoretical analysis mirrors practical complexities, even if it means accepting more conservative bounds for the sake of realism. Our focus in this work is to close this gap between theoretical analysis and practical implementations of AC algorithms.
\begin{table*}[ht]
\centering
\caption{This table summarizes the features of different actor-critic convergence results. Our result is the first to provide last iterate sample complexity results of AC for an MDP setting with multi layer neural network for the actor-critic, continuous state and action space, Markovian sampling  and global optimality results on the last iterate performance.
\vspace{2mm}}
\label{tbl_related}%
{\begin{tabular}{|c|c|c|c|c|c|c|}
\hline
References & \multicolumn{1}{c|}{\begin{tabular}[c]{@{}c@{}}  Per. of \\ \textbf{L}ast Iterate \end{tabular}} & \multicolumn{1}{c|}{\begin{tabular}[c]{@{}c@{}}\textbf{G}lobal \\ Optimality \end{tabular}} &  \multicolumn{1}{c|}{\begin{tabular}[c]{@{}c@{}} \textbf{C}ontinuous State \\ Action Space \end{tabular}} &   \multicolumn{1}{c|}{\begin{tabular}[c]{@{}c@{}} \textbf{M}ulti Layer  \\ NN AC \end{tabular}} &   \multicolumn{1}{c|}{\begin{tabular}[c]{@{}c@{}} \textbf{M}arkovian \\ Sampling \end{tabular}} &   \multicolumn{1}{c|}{\begin{tabular}[c]{@{}c@{}} Sample \\ Complexity \end{tabular}}   \\ \hline
     \citep{xu2020non}     &                    \textcolor{red}{\xmark}               &                \textcolor{green}{\cmark}         &     \textcolor{green}{\cmark}               &                     \textcolor{red}{\xmark}                                                                  &           \textcolor{green}{\cmark}    &       $\tilde{\mathcal{O}}(\epsilon^{-4})$                                                       \\ 
    \citep{khodadadian2021finite}     &                    \textcolor{red}{\xmark}               &                \textcolor{green}{\cmark}         &     \textcolor{red}{\xmark}               &                     \textcolor{red}{\xmark}                                                                  &            \textcolor{green}{\cmark}    &       $\tilde{\mathcal{O}}(\epsilon^{-3})$                                                       \\     
 \citep{xu2020improving}          &       \textcolor{red}{\xmark}               &                \textcolor{green}{\cmark}         &      \textcolor{green}{\cmark}                &                     \textcolor{red}{\xmark}                                                                  &           \textcolor{green}{\cmark}    &       $\tilde{\mathcal{O}}(\epsilon^{-3})$                                                        \\  
 \citep{xu2021doubly}            &        \textcolor{red}{\xmark}               &                \textcolor{green}{\cmark}         &     \textcolor{green}{\cmark}               &                     \textcolor{red}{\xmark}                                                                               &           \textcolor{green}{\cmark}    &       $\tilde{\mathcal{O}}(\epsilon^{-4})$                                                       \\  
 \citep{wang2019neural}          &        \textcolor{red}{\xmark}               &                \textcolor{green}{\cmark}         &     \textcolor{red}{\xmark}               &                     \textcolor{red}{\xmark}                                                                                &           \textcolor{red}{\xmark}    &       $\tilde{\mathcal{O}}(\epsilon^{-4})$                                                       \\
\citep{cayci2022finite}          &        \textcolor{red}{\xmark}               &                \textcolor{green}{\cmark}         &     \textcolor{red}{\xmark}               &                     \textcolor{red}{\xmark}                                                                                &           \textcolor{red}{\xmark}    &       $\tilde{\mathcal{O}}(\epsilon^{-4})$                                                       \\ 
\citep{fu2020single}            &        \textcolor{red}{\xmark}               &                \textcolor{green}{\cmark}         &     \textcolor{red}{\xmark}               &                    \textcolor{green}{\cmark}                                                                               &           \textcolor{red}{\xmark}    &       $\tilde{\mathcal{O}}(\epsilon^{-6})$                                                                           \\
\citep{tian2023convergence}     &        \textcolor{red}{\xmark}               &                \textcolor{red}{\xmark}         &     \textcolor{red}{\xmark}               &                    \textcolor{green}{\cmark}                                                                               &           \textcolor{green}{\cmark}    &       $\tilde{\mathcal{O}}(\epsilon^{-2})$                                         \\
\textbf{ This work}             &        \textcolor{green}{\cmark}               &               \textcolor{green}{\cmark}        &     \textcolor{green}{\cmark}               &                    \textcolor{green}{\cmark}                                                                               &           \textcolor{green}{\cmark}    &       $\tilde{\mathcal{O}}(\epsilon^{-3})$
\\
\hline
\end{tabular}
\footnotetext[1]{Footnote}
}
\footnotetext[1]{Footnote}
\end{table*}

To align closely with practical settings, it is essential that any theoretical analysis of Actor-Critic (AC) algorithms thoroughly considers five crucial aspects. These are (a) \textbf{M}ulti-layer neural network parametrization for actor/critic, (b) \textbf{M}arkovian sampling, (c) \textbf{C}ontinuous state-action spaces, (d) performance of the \textbf{L}ast iterate, and (e) \textbf{G}lobal optimality. The significance of the MMCLG criteria lies in its alignment with practical implementations: Deep neural networks are commonly used for actor-critic implementations \citep{lee2020stochastic}, data in real-world scenarios is typically sampled in a Markovian fashion \citep{8896945}, and most applications, like robotics, operate in continuous spaces \citep{dankwa2019twin}. Furthermore, in practice, the last iterate of the algorithms is used to evaluate performance \citep{9576103}. Finally for training neural networks, the global convergence truly matters, as local convergence can be misleading in evaluating the effectiveness of the trained network \citep{swirszcz2016local}. 

\ignore{\textcolor{red}{Mudit, can you refine this paragraph by combining papers which covers only subset of MMCLG criteria ??, I tried but it requires more efforts}} The existing literature (cf. Table \ref{tbl_related}), while extensive, does not have any works that simultaneously address the above-mentioned five dimensions (MMCLG). There exist works which have achieved local convergence (which upper bound the quantity $\frac{1}{T}\sum_{i=1}^{T}||{\nabla}J(\lambda_{t})||^{2}$) for finite state spaces for \textbf{M}ulti-layer settings  with \textbf{M}arkovian sampling \citep{tian2023convergence}]. Another set of results have obtained \textbf{G}lobal optimality results known as average iterate complexity bounds or regret bounds. These establish an upper bound on the quantity $\frac{1}{T}\sum_{i=1}^{T}\left(J^{*}-J(\lambda_{t})\right)$ known as the regret\ignore{\textcolor{red}{NEED TO REFINE}}. For a linear critic parametrization,  such \textbf{G}lobal optimality bounds have been established in \citep{xu2020improving} for \textbf{M}arkovian sampling. For a neural network actor and critic parameterization \citet{cayci2022finite}, establishes \textbf{G}lobal optimality  where the actor and critic neural networks have a single hidden layer, while \citet{fu2020single} does so for a \textbf{M}ulti Layer neural network of arbitrary depth. Notably, none of the above works focus on the last iterate convergence aspect, which requires an upper bound on $J^{*}-J(\lambda_{T})$. Also none of the works having a \textbf{M}ulti Layer neural network of arbitrary depth work for \textbf{C}ontinuous state-action spaces.  Hence, in this work, we ask this question


{ \emph{Is it possible to develop a theoretical analysis of actor-critic algorithm that covers all MMCLG criteria in one analysis?}}

\noindent The above question essentially implies that can we obtain \textbf{G}lobal sample complexity bounds for the \textbf{L}ast iterate convergence of the actor-critic algorithm with a \textbf{M}ulti-layer neural network parametrization of the critic, without assuming i.i.d. sampling (under \textbf{M}arkovian) for \textbf{C}ontinuous spaces? We answer the above question in the affirmative in this work. Our main contributions as listed as follows.

\begin{itemize}[leftmargin=*]

\item  We establish an upper bound on the performance of the \textbf{L}ast Iterate  $J^{*}-J(\lambda_{T})$ in terms of the sum of the errors incurred in the estimation of the critic. This is done using our novel analysis combining the smoothness assumption on the policy parametrization and the weak gradient bound condition for MDP. This analysis is different from existing works such as  \citep{fatkhullin2023stochastic}, \citep{masiha2022stochastic}, where an upper bound on the last iterate performance is obtained using an unbiased estimator of the  policy gradient using sample trajectories. This is not available to us in actor critic where a parametric estimate of the critic is used and we have to account for the error incurred in critic estimation. Our analysis does not rely on the cardinality of the action space, unlike in existing multi layer critic analyses such as \citet{fu2020single} thereby, our \textbf{G}lobal convergence bound works for \textbf{C}ontinous state-action spaces.

In the analysis of critic error estimation, we derived a novel decomposition of error. The error is split into the error incurred due to the limited approximation ability of the class of function representing the critic and the error incurred due to the limited sample size to estimate the critic, as well as the error incurred in solving the critic estimate in a finite number of steps. This decomposition allows us to consider the \textbf{M}arkov dependence of samples and  a \textbf{M}ulti-layer neural network parametrization of the critic, which is the first result that achieves this. This is in contrast to  \citep{cayci2022finite,fu2020single}, where i.i.d. sampling was assumed.   

\item We derive a last iterate convergence sample complexity bound of $\tilde{\mathcal{O}}(\epsilon^{-3})$ for the actor-critic algorithm with neural network parameterizations for the critic and the actor. True to our knowledge, this work is the first to present a last iterate global convergence result for an actor-critic algorithm with neural network critic parametrization. It is also the best sample complexity for global convergence in terms of $\epsilon$ for an actor-critic algorithm with a neural network parametrization of the critic.

\end{itemize}

\section{Related Works} \label{Related Works}

 \textbf{Policy Gradient:} Policy gradient algorithms, first conceptualized in  \cite{sutton1999policy} perform a gradient step on the parameters to obtain an estimate of the optimal policy. The estimate of the action value (or advantage) function is obtained by following the current estimate of the policy and calculating the action value function from the obtained rewards. In such a case, sample complexity estimates are possible without the need to assume parametric form of the action value function as is done in \citet{agarwal2020optimality}. It obtained a sample complexity bound of $\tilde{\mathcal{O}}\left(\frac{1}{\epsilon^{4}}\right)$.  Further improvements have been obtained in \citep{liu2020improved,mondal2023improved}, where the proposed algorithm in \citep{mondal2023improved} achieves a sample complexity of $\tilde{\mathcal{O}}\left(\frac{1}{\epsilon^{2}}\right)$. Note that these results were all average iterate convergence results. Works such as \citep{fatkhullin2023stochastic,masiha2022stochastic} have obtained last iterate convergence of policy gradient algorithms.
 

\textbf{Actor Critic:} Actor critic methods aim to combine the benefits of the policy gradient methods and $Q$-learning based methods. Local convergence results results for Actor Critic were obtained in \citep{castro2010convergent} and \citep{maei2018convergent}. Average \ignore{\bf check - best or average?} iterate convergence results actor critic using a linear critic have been obtained in \citet{xu2020non} and \citet{xu2020improving} with sample complexity of $\tilde{\mathcal{O}}\left(\frac{1}{{\epsilon}^{4}}\right)$ and   $\tilde{\mathcal{O}}\left(\frac{1}{{\epsilon}^{3}}\right)$\footnote{ This work shows a sample complexity of $\epsilon^{-2}$ in the published version, which was corrected to $\epsilon^{-3}$ in a later arXiv version.}, respectively. Average iterate \ignore{\bf check - best or average?} convergence for Actor Critic  where neural networks are used to represent the actor and critic are obtained in works such as \citep{gaur2023global,wang2019neural} which obtain sample complexities of $\tilde{\mathcal{O}}\left(\frac{1}{{\epsilon}^{4}}\right)$ and $\tilde{\mathcal{O}}\left(\frac{1}{{\epsilon}^{6}}\right)$ respectively \ignore{\bf what are the results in this?}. More recently works such as \citet{tian2023convergence}  obtain local convergence \ignore{\bf when previous line is global - why is local improved?} results with a sample complexity of $\tilde{\mathcal{O}}\left(\frac{1}{{\epsilon}^{2}}\right)$, with \ignore{\bf Mudit, all results?} a finite state and action space, limiting the practicality of the algorithm.


\textbf{Last Iterate Convergence:} In optimization literature, the strongest convergence bound \ignore{\bf I do not see why this is strongest when it does not imply the others?} for an algorithm that can be obtained is known as a last iterate convergence bound. It can be written as 
\begin{equation}
     f(\lambda^{*}) - f(\lambda_{t})  \le {\mathcal{O}(h(t))}.
\end{equation}
Here, $f$ is the objective function of interest to be maximized. $\lambda_{t}$ is the parameter obtained at the $t^{th}$ iteration of the algorithm and $\lambda^{*}$ is the optimal parameter corresponding to the highest possible value of the objective function and $h$ is some function of the number of iteration and possibly the sample size needed at each iteration. Such results were typically proven for gradient descent-type algorithms for convex objective functions \citep{citeulike:163662,balkanski2017sample}. 
In many modern machine learning applications, the objective function of interest is non-convex. This means many convergence results only demonstrate a local convergence, which can be written as  $\frac{1}{t}  \sum_{i=1}^{t} ||{\nabla}f(\lambda_{i})||^{2}  \le \tilde{\mathcal{O}(h(t))}$, 
which only guarantees that the algorithm will converge to a local optimum. Example of such results are in works such as \cite{li2019convergence,pmlr-v189-chen23b} and \citet{tian2023convergence}.

For value iteration based methods such as $Q$ learning and fitted $Q$ iteration, last iterate convergence results with neural network parametrization have been obtained in works such as \citet{cai2019neural} and \citet{pmlr-v202-gaur23a} respectively.

In policy iteration methods such as gradient methods, with additional `compatible function approximation' assumptions \citet{sutton1999policy}, an upper bound is established on the \textit{regret} defined as

\vspace{-.2in}
\begin{equation}
     \frac{1}{t}  \sum_{i=1}^{t}(f(\lambda^{*}) - f(x_{i}))  \le \tilde{\mathcal{O}(h(t))}
\end{equation}

Such a result is the only type of global convergence shown so far for actor critic methods.


\ignore{\bf Check, why we need to define the best one - I removed from Intro - since that is weaker and if average is there, no point giving best.} 

In order to establish an upper bound on $f(\lambda^{*}) - f(\lambda_{t})$ for non-convex $f(x)$,   \citet{Polyak1963GradientMF} established the notion of \textit{weak gradient bound} defined as 

\vspace{-.2in}
\begin{equation}
        {\mu}\cdot||{\nabla}f(x)||^{\alpha}   \le f(\lambda^{*}) - f(\lambda)  
\end{equation}

where $\mu$ is a positive constant that depends on the function $f$ and $\alpha \in [1,2]$. We note that using this condition,  last iterate convergence results have been demonstrated for non-convex optimization in works such as \citet{pmlr-v195-yue23a} and  \citet{9769873}. For an MDP setup, under standard assumptions, a PL like condition was established in \citet{ding2022global}. This has been used in policy gradient works such as \citep{masiha2022stochastic,fatkhullin2023stochastic} to establish last iterate convergence. 


\ignore{\bf Paper looks too simplistic for clear reject. PL type condition has been used in other works for convex - you use Ding paper for such condition and get results. I do not see in the entire paper any real challenge either in Intro or here. }

\textbf{I.I.D. vs Markov Sampling in AC:} Prior analyses of AC algorithms with neural network actor and critic algorithms such as \citet{wang2019neural,cayci2022finite,fu2020single} all rely on \textit{local linearization} of the neural networks. These techniques assume that the samples are drawn independently from the stationary distribution of a fixed policy. However, as is shown in \citet{mnih2013playing}, Q-learning algorithms with deep neural networks require the use of \textit{target networks} and  \textit{experience replay} to converge. \citet{lillicrap2015continuous} showed the same is the case for actor critic algorithms particularly for cases with continuous action spaces. Thus, the sampling approach (which is iid) in existing analyses for AC with neural network parametrizations is not applicable in practice, and we need convergence analysis under Markovian sampling to better reflect the practical cases in theory. 

\section{Problem Formulation} \label{Problem Setup}

We consider a discounted Markov Decision Process (MDP) given by the tuple ${\mathcal{M}}:=(\mathcal{S}, \mathcal{A}, P, R, \gamma)$, where $\mathcal{S}$ is a bounded measurable state space, $\mathcal{A}$ is the set of actions which is also a bounded measurable space. Note that for our setup, both the state and action space can be infinite. ${P}:\mathcal{S}\times\mathcal{A} \rightarrow \mathcal{P}(\mathcal{S})$  is the probability transition function, $r: \mathcal{S}\times\mathcal{A} \rightarrow ([0,1])$ is the reward function on the state action space  and $0<\gamma<1$ is the discount factor. A policy $\pi:\mathcal{S} \rightarrow \mathcal{P}(\mathcal{A})$ maps a state to a probability distribution over the action space. The action value function for a given policy $\pi$ is given by 
\begin{eqnarray}
Q^{\pi}(s,a) = \mathbb{E}\left[\sum_{t=0}^{\infty}\gamma^{t}r(s_{t},a_{t})|s_{0}=s,a_{0}=a\right],    
\end{eqnarray}
where  $a_{t} \sim \pi(\cdot|s_{t})$ and $s_{t+1} \sim P(\cdot|s_{t},a_{t})$ for $t=\{0,\cdots,\infty\}$. For a discounted MDP, we  define the optimal
action value functions as
\begin{eqnarray}
Q^{*}(s,a) =  \sup_{\pi}Q^{\pi}(s,a),  \hspace{0.5cm} \forall (s,a) \in  \mathcal{S}\times\mathcal{A} \label{ps_2}  .  
\end{eqnarray}
A policy that achieves the optimal action-value functions is known as the  \textit{optimal policy} and is denoted as $\pi^{*}$. 
We define $\rho_{\nu}^{\pi}(s)$ as the stationary state distribution induced by the policy $\pi$ starting at state distribution $\nu$ and $\zeta_{\nu}^{\pi}(s,a)$ is the corresponding stationary state action distribution defined as $\zeta_{\nu}^{\pi}(s,a)=\rho_{\nu}^{\pi}(s){\cdot}\pi(a|s)$. 
We can define the state action visitation distribution as $ d^{\pi}_{\nu}(s,a) = (1-\gamma)\sum_{t=0}^{\infty} {\gamma}^{t}Pr^{\pi}(s_{t}=s,a_{t}=a|(s_{0},a_{0}) \sim \nu)$, where $Pr^{\pi}(s_{t}=s,a_{t}=a|(s_{0},a_{0}) \sim \nu)$ denotes the probability that the state action pair at time $t$ is $(s,a)$ when following the policy $\pi$ with starting state action distribution of $\nu$.

We additionally define the Bellman operator for a policy $\pi$ on a function $Q:\mathcal{S}\times\mathcal{A} \rightarrow \mathcal{S}\times\mathcal{A}$ is defined as 
\begin{eqnarray}
(T^{\pi}Q)(s,a) &=& r(s,a) \nonumber\\&&+ \gamma\int{Q}(s',\pi(s'))P(ds'|s,a)    \label{ps_3}
\end{eqnarray}
Further, operator $P^{\pi}$ is  defined as 
\begin{eqnarray}
P^{\pi}Q(s,a)=\mathbb{E}[Q(s',a')|s' \sim P(\cdot|s,a), a' \sim \pi(\cdot|s')]  
\end{eqnarray}

This is the one step Markov transition operator for policy $\pi$ for the Markov chain defined on $\mathcal{S}\times\mathcal{A}$ with the transition dynamics given by $S_{t+1} \sim P(\cdot|S_{t},A_{t})$ and $A_{t+1} \sim \pi(\cdot|S_{t+1})$. It defines a distribution on the state action space after one transition from the initial state. Similarly, $P^{\pi_{t}}P^{\pi_{t-1}}\cdots{P}^{\pi_{1}}$ is the $m$-step Markov transition operator following policy $\pi_t$ at steps $1\le t\le m$. 

\textbf{Neural Network Parametrization.} We define a neural network of $D$ layers with $m$ neurons per layer as follows
\begin{align}
    y = &\frac{1}{\sqrt{m}} {\cdot} \sigma(b_{D}^{T}x_{D-1})  \label{net_1}\\
    x^{h} =& \frac{1}{\sqrt{m}} {\cdot} \sigma(W_{h}^{T}x_{h-1})   \label{net_2}, \\
    h  \in& \{1,\cdots,D-1\}, \label{net_3}
\end{align}
where $m$ is the number of neurons in each neural network layer, $D$ represents the number of layers in the neural network, $W_{h}$ represents the weight matrix for the $h^{th}$ layer, where $b_{D}$ represents the weight vectors of the final. Also, $x_{0}$ is the input to the neural network, $\sigma$ is an activation function, typically a function such as sigmoid or ReLU may be used here.
We denote the neural network representing the critic corresponding to a set of parameters $\theta = \{W_{1},\cdots,W_{D-1},b_{D}\}$ for a given state action pair as $(s,a)$ as $Q_{\theta}(s,a)$. For notational convenience we denote the input to this neural networks as  $(s,a) \in \mathcal{S}\times\mathcal{A}$.
For the actor-network, we use a Gaussian policy given by $\pi_{\lambda}(s) = \mathcal{N}(\mu_{\lambda_{1}}(s),\kappa_{\lambda_{2}}(s))$. Here $\mu_{\lambda_{1}}(s)$  and $\kappa_{\lambda_{2}}(s)$  represent the mean and variance of a normal distribution represented by neural networks with the state $s \in \mathcal{S}$ as an input. We denote the set of parameters as $\lambda = \{\lambda_{1},\lambda_{2}\}$. The activation functions for the actor network are required to be smooth. Both these networks has the structure defined in equations \eqref{net_1} to \eqref{net_3}. Actor and critic networks sharing the same architecture is implemented in practical Actor Critic implementations like \citet{9066984}.

\vspace{-.1in}
\section{Actor Critic Algorithm} \label{Proposed Algorithm}
 In an actor-critic algorithm  \citep{konda1999actor}, the aim is to maximize the expected return \ignore{\bf Find exact terminology from literature for this function} given by 
 \begin{equation}
     J(\lambda) = \mathbb{E}_{s \sim \nu, a \sim \pi_{\lambda}(.|s)}Q^{\pi_{\lambda}}(s,a) 
 \end{equation}
 
A policy gradient step is performed to update  the policy parameters of the actor. For our setup, the policy is parameterized as $\{\pi_{\lambda} , \lambda \in \Lambda\}$ and $\Lambda \subset \mathbb{R}^{d}$ where $d$ is a positive integer. We have $K$ total iterations of the Algorithm. At iteration $k$, the policy parameters are updated using a natural policy gradient step given by
\begin{eqnarray}
    \lambda_{k+1} = \lambda_{k} + {\alpha_{k}}{\nabla}_{\lambda}J(\lambda_{k}), \label{alg_eq_2}
\end{eqnarray}
From the policy gradient theorem in \citep{sutton1999policy} we have 
\begin{eqnarray}
   {\nabla}_{\lambda}J(\lambda_{k}) = \mathbb{E}_{s,a}(\nabla_{\lambda}{\log(\pi_{\lambda_{k}}(a|s))}Q^{\pi_{\lambda_{k}}}(s,a)),
   \label{alg_eq_3_1}
\end{eqnarray}
where $(s,a) \sim d^{\pi_{\lambda_{k}}}_{\nu}$. This policy update requires us to calculate the $Q$ function for the current estimate of the optimal policy. Thus, we maintain a parameterized estimate of the $Q$-function, which is updated at each step and is used to approximate $Q^{\pi_{\lambda_{k}}}$. An estimate of its parameters is obtained by solving an optimization of the form 
\vspace{-.1in}
\begin{eqnarray}
    \argminA_{\theta \in \Theta} \mathbb{E}_{s,a}({Q^{\pi_{\lambda_{k}}}-Q_{\theta}})^{2}, \label{alg_eq_6}
\end{eqnarray}
where  $(s,a) \sim d^{\pi_{\lambda_{k}}}_{\nu}$, $\Theta$ is the space of parameters for the critic neural networks, and $Q_{\theta}$ is the neural network corresponding to the parameter $\theta$. This step is known as the critic step. 
\begin{algorithm}[t]
	\caption{Actor Critic  with Neural Parametrization}
	\label{algo_1}
	\textbf{Input:} $\mathcal{S},$ $ \mathcal{A}, \gamma, $ Time Horizon  $K \in \mathcal{Z}$ , sample batch size $n \in \mathcal{Z}$, resample batch size $L \in \mathcal{Z}$, Updates per time step  $ J \in \mathcal{Z}$  ,starting state sampling 
    distribution $\nu$, Actor step size $\alpha$, Critic step size $\beta^{'}$, starting actor parameter $\lambda_{1}$,  
    \begin{algorithmic}[1]
    
		\FOR{$k\in\{1,\cdots,K\}$} 
		{  
        \STATE Sample $n$ tuples $(s,a,r,s^{'})$  by following the policy $\pi^{\lambda_{k}}$ from a starting state distribution $\nu$ and store the tuples. 
            \FOR{$j\in\{1,\cdots,J\}$} 
		  {
             \STATE Initialize $\theta_{0}$ using a standard Gaussian for all elements of weight matrices and $b_{D}$ from Unif$(-1,1)$ for the bias vector.
            \FOR{$i\in\{1,\cdots,L\}$} 
		  {
		      \STATE Sample a tuple $(s_{i},a_{i},r_{i},s^{'}_{i})$ with equal probability from the stored dataset \label{a1_l1}\\
                \STATE Sample $a^{'}_{i}$ using $\pi^{\lambda_{k}}(.|s^{'}_{i})$
                \STATE Set $y_i= r_{i} + {\gamma}Q_{k,j-1}(s^{'}_{i},a^{'}_{i})$, \label{a1_l2}\\ 
                \STATE $\theta^{'}_{i} = \theta_{i-1} + {\beta^{'}}(y_{i} - Q_{\theta_{i}}(s_{i},a_{i})){\nabla}Q_{\theta_{i}}(s_{i},a_{i})  $ \label{a1_l3}\\
                \STATE $\theta_{i} = \Gamma_{\theta_{0},\frac{1}{(1-\gamma)}}\left(\theta^{'}_{i}\right)$
             }
             \ENDFOR\\
             \STATE $Q_{k,j} = Q_{{\theta^{'}}}$  where ${\theta^{'}} =\frac{1}{L}\sum_{i=1}^{L}(\theta_{i}) $ \\
             }
             \ENDFOR\\
             \STATE $d_{k} = \frac{1}{n}\sum_{i=1}^{n}{\nabla}{\log}(\pi(a_{i}|s_{i}))Q_{k,J}(s_{i},a_{i})$\\
            \STATE Update $\lambda_{k+1} =  \lambda_{k}  + \left(\frac{\alpha}{k}\right)\frac{d_{k}}{||d_{k}||} $
		}
		\ENDFOR\\
	Output: $\pi_{\lambda_{K+1}}$
	\end{algorithmic}
\end{algorithm}

We summarize the  Actor-Critic approach in Algorithm \ref{algo_1}. It has one outer for loop indexed by the iteration counter $k$. The first inner for loop indexed by $j$ is the loop where the critic step is performed.  At a fixed iteration $k$ of the main for loop and iteration $j$ of the first inner for loop, we solve the following optimization problem 

\begin{eqnarray}
\argminA_{\theta \in \Theta} \mathbb{E}_{s,a}({T^{\pi_{\lambda_{k}}}Q_{k,j-1}(s,a)-Q_{\theta}}(s,a))^{2},    
\end{eqnarray}
 Where $(s,a) \sim d^{\pi_{\lambda_{k}}}_{\nu}$. This is the \textit{target network} technique. For the inner loop at iteration $j$, the target is  $T^{\pi_{\lambda_{k}}}Q_{k,j-1}(s,a)$. The first inner for loop has a nested inner for loop indexed by $i$ where the optimization step for the current target is performed. Note that in line 10 $\Gamma_{\theta_{0},(1-\gamma)^{-1}}$ represents the projection operator on the ball of radius $(1-\gamma)^{-1}$ centered on $\theta_{0}$ in the space $\Theta$. Formally this set is defined as  $ \Theta^{'} = \{\theta \in \Theta: ||W_{h}-W^{0}_{h}|| \le (1-\gamma)^{-1}, \forall h \in \{1,\cdots,D-1\} \}$.
 
Here the tuples $(s_{i},a_{i},r_{i},s^{'}_{i})$ are sampled randomly from the stored dataset and the optimization is performed using this sampled data. This random sampling on line 6 is the \textit{experience replay} technique. The target network is updated in the second inner for loop indexed by $i$. The inner loop indexed by $j$ controls how many times the target network is updated. 

The algorithm laid out here is a simplified version of the Twin Delayed Deep Deterministic Policy Gradient algorithm \citep{10.1145/3387168.3387199}.

\section{Theoretical Analysis: Global Convergence} \label{Main Result}
\subsection{Assumptions} \label{Assumptions}
Before stating the main result, we formally describe the required assumptions in this subsection.
\begin{assump}  \label{assump_1} 
For any $\lambda, \lambda_{1}, \lambda_{2} \in \Lambda$ and $(s,a) \in (\mathcal{S}\times\mathcal{A})$ we have 
%

   (i) $\|{\nabla}log(\pi_{\lambda_{1}})(a|s) - {\nabla}log(\pi_{\lambda_{2}})(a|s)\|  \le  \beta\|\lambda_{1} -\lambda_{2}\|$,
   
   (ii) $\|{\nabla}log(\pi_{\lambda_{1}})(a|s)\| \le M_{g}$,  
   
   (iii) $\mathbb{E}_{(s,a) \sim d^{\pi_{\lambda_{1}}}_{\nu}}({\nabla}{\log}\pi_{\lambda_{1}}(a|s))({\nabla}{\log}\pi_{\lambda_{1}}(a|s))^{T}  \succcurlyeq \mu_{f}I_{d}$
   
where $\beta, M_{g}, \mu_{f} \ge 0$.
\end{assump}
Such assumptions have been utilized in prior policy gradient based works such as \citep{masiha2022stochastic,fatkhullin2023stochastic} and actor critic using linear critic such as \citet{xu2020improving}. and global convergence results for neural critic parameterization such as \citet{fu2020single,wang2019neural}, which restrict their analysis to energy-based policies for finite action spaces. Works such as \citet{cayci2022finite} only considers soft-max policies which are a simplification of energy based policies. However, as \citet{fatkhullin2023stochastic} points out, soft-max and energy based policies  do not always satisfy $(iii)$, while Gaussian policies are shown to satisfy this assumption for sufficiently deep neural networks.  \ignore{\bf INCOMPLETE} \ignore{\bf  more general than only 1 paper?}

\begin{assump}  \label{assump_4} 
For any $\lambda\in \Lambda$, let $\pi_{\lambda}$ be the corresponding policy, $\nu$ be the starting distribution over the state action space, and let $\zeta_{\nu}^{\pi_{\lambda}}$ be the corresponding stationary state action distribution. We assume that there exists a positive integer $p$ such that for every positive integer $\tau$ and for any set ${\cdot} \in \mathcal{S}\times\mathcal{A} \nonumber$ and any $(s) \in \mathcal{S} \nonumber$
\begin{align}
    d_{TV}\left(\mathbb{P}((s_{\tau},a_{\tau}) \in {\cdot} |(s_{0})=(s)),\zeta_{\nu}^{\pi_{\lambda}}({\cdot})\right) \leq p{\rho}^{\tau}, \nonumber
\end{align}
\end{assump}
This assumption implies that the Markov chain is geometrically mixing. Such assumption is widely used both in the analysis of stochastic gradient descent literature such as \citet{9769873,sun2018markov}, as well as finite time analysis of RL algorithms such as  \citet{xu2020improving}. In \citet{fu2020single}, it is assumed that data can be sampled from the stationary distribution of a given policy. Instead, in practice we can only sample from a Markov chain which has a stationary distribution as the desired distribution to sample from.

\begin{assump}  \label{assump_5} 

For any fixed $\lambda \in \Lambda$  we have 
\begin{eqnarray} 
  \mathbb{E}\left(A^{\pi_{\lambda}}(s,a) -(1-\gamma){w^{*}(\lambda)}^{\top}{\nabla}\log(\pi_{\lambda})(a|s)\right)^{2}  \nonumber\\
  \le \epsilon_{bias} \nonumber
\end{eqnarray}
Here, the expectation is over $(s,a) \sim d^{\pi^{*}}_{\nu}$ where $\pi^{*}$ is the optimal policy. We also have $w^{*}(\lambda) = F(\lambda)^{\dagger}{\nabla}J(\lambda)$ where $F(\lambda) = \mathbb{E}_{(s,a) \sim d^{\pi_{\lambda}}_{\nu}}({\nabla}{\log}\pi_{\lambda}(a|s))({\nabla}{\log}\pi_{\lambda}(a|s))^{T}$ . 
\end{assump}

\ignore{\textcolor{red}{THIS CANNOT GO OUTSIDE THE MARGINN}\ignore{\bf something missing}} \citet{wang2019neural} proves that this error goes to zero if both the actor and critic are represented by overparametrised neural networks. This assumption allows us to establish the weak gradient bound property for our MDP setup. It is used in policy gradient works such as \citep{yuan2022general,masiha2022stochastic,fatkhullin2023stochastic} to establish last iterate convergence.   \ignore{\bf Last paper was not last iterate - so why they used? Also, if this assumption already does the needed, then any novelty from best to average?}
\begin{assump}  \label{assump_7} 
For any fixed $ \lambda \in \Lambda$ we have
\begin{eqnarray} 
  \min_{\theta_{1} \in \Theta^{'}}\mathbb{E}_{s,a \sim \zeta_{\nu}^{\pi_{\lambda_{k}}}}\left(Q_{\theta_{1}}(s,a) - T^{\pi_{\lambda}}Q_{\theta}(s,a) \right)^{2} \le \epsilon_{approx} \label{assump_7_1} \nonumber
\end{eqnarray}
\end{assump}

 This assumption ensures that a class of neural networks are able to approximate the function obtained by applying the Bellman operator to a neural network of the same class. Similar assumptions are taken in \citep{fu2020single,wang2019neural}. In works such as \citep{cayci2022finite}, stronger assumptions are made wherein the function class used for critic parametrization is assumed to be able to approximate any smooth function.   \ignore{\bf only one unpublished paper - sounds not good. }

Before we move on to the main result, we want to state the key lemma proved in \citet{ding2022global}, that is used to obtain the last iterate convergence.

\begin{lemma}  \label{lem_main} 
If Assumptions \ref{assump_1} and \ref{assump_5} hold then for any fixed  $ \lambda \in \Lambda$ we have
\begin{eqnarray}
    \sqrt{{\mu}}(J(\lambda^{*}) -  J(\lambda)) \le \epsilon^{'} + \|{\nabla}J(\lambda)\|, \nonumber
\end{eqnarray}

where $\epsilon^{'} = \frac{{\mu}_{f}{\sqrt{\epsilon_{bias}}}}{M_{g}(1-\gamma)}$ and $\mu = \frac{\mu^{2}_{f}}{2M^{2}_{g}}$
\end{lemma}

\subsection{Main Result} \label{Theorem Statement}
Next, we present the main result. 

\begin{thm} \label{thm}
Suppose Assumptions \ref{assump_1}-\ref{assump_7} hold and we have  $\alpha = \frac{7}{2\sqrt{{\mu}}}$ and $\beta^{'} = \frac{1}{\sqrt{L}}$ then from Algorithm \ref{algo_1} we obtain 
\begin{align}
   J(\lambda^{*}) - J(\lambda_{t}) \le&   \tilde{{\mathcal{O}}}\left(\frac{1}{{t}}\right)  +  \tilde{{\mathcal{O}}}\left(\frac{1}{\sqrt{n}}\right) \nonumber\\ 
     &+   \tilde{{\mathcal{O}}}\left(\frac{1}{{L}^{\frac{1}{4}}}\right)  +   \tilde{{\mathcal{O}}}(m^{-\frac{1}{12}}D^{\frac{7}{2}})  \nonumber\\
     &\quad 
     +   \tilde{{\mathcal{O}}}(\gamma^{J}) + \tilde{{\mathcal{O}}}(\sqrt{\epsilon_{bias}}) \nonumber\\
     &\quad \quad +   \tilde{{\mathcal{O}}}(\sqrt{\epsilon_{approx}}). \nonumber
\end{align}
Hence, for $t  =  \tilde{\mathcal{O}}(\epsilon^{-1}) $, $J = \tilde{\mathcal{O}}\left(\log\left(\frac{1}{\epsilon}\right)\right) $, $n = \tilde{\mathcal{O}}\left(\epsilon^{-2}\right)$, $ L=\tilde{\mathcal{O}}\left(\epsilon^{-4}\right)$ we have
\begin{align}
 J(\lambda^{*}) - J(\lambda_{t})  \leq& \epsilon +  \tilde{\mathcal{O}}(m^{-\frac{1}{12}}D^{\frac{7}{2}}) \nonumber\\
                                  &+    \tilde{\mathcal{O}}\left(\sqrt{\epsilon_{bias}} \right) + \tilde{\mathcal{O}}\left(\sqrt{\epsilon_{approx}} \right),  \nonumber                                
\end{align}
which implies a sample complexity of $  t {\cdot} n = \tilde{\mathcal{O}}\left({\epsilon^{-3}}\right)$. 
\end{thm}

The fourth term is the consequence of the finite size of the critic neural network. Such terms are present in other results where multi layer neural network parametrizations are used such as \citet{fu2020single} and \citet{tian2023convergence}. We can see that as the width of the neural network tends to infinity these terms go to zero. This is in keeping with Neural Tanget Kernet (NTK) \ignore{\bf full form?} theory \citet{jacot2018neural}, which states that in the infinite width limit neural networks converge to linear functions.

\section{Proof Sketch of Theorem \ref{thm}} 

The proof is split into two stages. In the first stage, we show how to obtain the last iterate performance gap as a function of the errors incurred in estimating the critic at each step. The second part decomposes the critic estimation error into its constituent components, which are bounded separately.

{\bf Upper Bounding Last Iterate Performance Gap: } Under Assumption \ref{assump_1}, \citet{yuan2022general} proves that the expected return is a smooth function. Thus, we have
\vspace{-.1in}
\begin{align}
    -J(\lambda_{t+1}) &\ge -J(\lambda_{t}) - \langle {\nabla}J(\lambda_{t}),\lambda_{t+1} - \lambda_{t} \rangle  \nonumber\\
    &+ {L_{J}}||\lambda_{t+1} - \lambda_{t}||^{2}. \label{main_res_0}
\end{align}
%
%
%
Here, $L_{J}$ is the smoothness parameter of the expected return and $\lambda_{t}$ denotes the critic parameters ate iteration $t$ of Algorithm \ref{algo_1} \ignore{\bf Equation has theta, and you are defining lambda - unclear}. 
From this, using Assumption \ref{assump_4}, the weak gradient domination property proved in \citet{ding2022global}, we obtain the following 
\vspace{-.1in}
\begin{align}
 J^{*} -J(\lambda_{t+1}) \le&   \left(1 -  \frac{{\eta_{t}}{\sqrt{\mu^{'}}}}{3}\right)(J^{*} -J(\lambda_{t})) \nonumber\\
 +& \frac{\alpha}{t}||{\nabla}J(\lambda_{t}) - d_{t} || \nonumber\\ 
+& {L_{J}}||\lambda_{t+1} - \lambda_{t}||^{2} + \frac{\alpha}{t}\epsilon^{'}. \label{main_res_1_1}
\end{align}
\ignore{\bf I do not see how this happens - I look at both equations, don't seem to get this. Need more details here. Also, how did theta change to lambda??}%
The key step now is recursively applying this condition and and substitute the result back in Equation \eqref{main_res_1_1}. We thus obtain the following result.
\begin{align}
  J(\lambda^{*}) - J(\lambda_{t})  \le&   \frac{\alpha.M_{g}}{t}\sum^{k=t-2}_{k=0}(\mathbb{E}|Q^{\lambda_{k}}(s,a) -  Q_{k,J}(s,a)|) \nonumber\\
  &+    \left(\frac{1}{t}\right) \left(J(\lambda^{*}) - J(\lambda_{2})\right) \nonumber\\
  &\quad  +  \frac{L_{J}{\alpha}^{2}}{t} + \mathcal{O}(\sqrt{\epsilon_{bias}}) .  \label{main_res_1}
\end{align}
Here, we have used the identity  $\nabla{J(\lambda_{k})}(s,a) = \mathbb{E}_{(s,a) \sim d^{\pi_{\lambda_{k}}}_{\nu}}\nabla\log({\lambda}_{k}(a|s))Q^{{\lambda}_{k}}(s,a)$. The details of this are given in Appendix \ref{thm proof}. 
\ignore{\bf unclear - how did $Q$ enter here in recursion - when previous thing did not have this. Cannot have outline with such big gaps. }

We thus obtain an upper bound on the last iterate optimality gap in terms of the average of the difference between the true gradient $\nabla{J(\lambda_{t})}$ and our estimate of the true gradient $d_{t}$ so far. Thus far we obtained an upper bound on the \textbf{G}lobal optimality of the \textbf{P}erformance of the last iterate. Since we do not have any terms that are functions of the cardinality of the state and action space, the analysis is valid for a \textbf{C}ontinous Action Space. Since the only restriction on the policy parametrization is smoothness, this analysis holds for a \textbf{M}ulti Layer Neural Network actor parametrization.  
  
{\bf Upper Bounding Error in Critic Step: }  

The critic error at each step is equivalent to solving the following optimization problem
\vspace{-.1in}
\begin{eqnarray}
    \argminA_{\theta \in \Theta^{'}} \mathbb{E}_{s,a}({Q^{\pi_{\lambda_{k}}}-Q_{{\theta}}})^{2} \label{main_res_4},
\end{eqnarray}
where $(s,a) \sim d^{\pi_{\lambda_{k}}}_{\nu}$. We denote as $Q_{k,J}$ as our estimate of  $Q^{\pi_{\lambda_{k}}}$ at $k^{th}$ iteration of Algorithm \ref{algo_1} and iteration $j$ of the first inner for loop. We obtain the following result for the for the action value function $Q$   
%
\begin{align}
    \mathbb{E}_{s,a}|Q^{\pi_{\lambda_{k}}}-Q_{k,J}| \leq&  \sum_{j=0}^{J-2} \gamma^{J-j-1}(P^{\pi_{\lambda_{k}}})^{J-j-1}{\mathbb{E}}|\epsilon_{k,j+1}| \nonumber\\
    &+ \tilde{\mathcal{O}}(\gamma^{J}) , \label{main_res_9}
\end{align}
where $\epsilon =  T^{\pi_{\lambda_{k}}}Q_{k,j-1} - Q_{k,j}$ is error incurred at iteration $j$ of the first inner for loop and iteration $k$ of the outer for loop of Algorithm \ref{algo_1}. In doing so, we have split the error incurred in estimating the critic at each step into the errors in estimating the target function at each iteration of the inner loop indexed by $j$. The first term on the right hand side of Equation \eqref{main_res_9}  denotes this error, in works such as \citet{farahmand2010error} this is known as the algorithmic error. The second term on the right hand side is called as the statistical error, which is  the error incurred due to the random nature of the system. Intuitively, the error in estimating the target function depends on how much data is collected at each iteration, how many samples we take in the buffer replay step and how well our neural network function class can approximate $T^{\pi_{\lambda_{k}}}Q_{k,j-1}$, i.e., the target function. Building upon this intuition, we split $\epsilon_{k,j}$ into four different components as follows.
\begin{align}
    \epsilon_{k,j} =& T^{\pi_{\lambda_{k}}}Q_{k,j-1} - Q_{k,j} \nonumber\\
                 =& \underbrace{T^{\pi_{\lambda_{k}}}Q_{k,j-1}-Q^{1}_{k,j}}_{\epsilon^{1}_{k,j}} + \underbrace{Q^{1}_{k,j} -Q^{2}_{k,j}}_{\epsilon^{2}_{k,j}} \nonumber\\
                 &+ \underbrace{Q^{2}_{k,j} -Q^{3}_{k,j}}_{\epsilon^{3}_{k,j}} +\underbrace{Q^{3}_{k,j} - Q_{k,j}}_{\epsilon^{4}_{k,j}} \nonumber\\
                 =& {\epsilon^{1}_{k,j}} + {\epsilon^{2}_{k,j}} + {\epsilon^{3}_{k,j}} + {\epsilon^{4}_{k,j}} .\label{last}
\end{align}
The first two components are dependent on the approximating power of the class of neural networks. The third term is dependent on the number of samples collected from the policy, for which the corresponding $Q$ function is to be measured. This is the term that will account for the \textbf{M}arkovian Sampling. The last term is dependent on the number of samples in the buffer replay step and accounts for the \textbf{M}ulti Layer Neural Network critic parametrization.

 We now define the various $Q$-functions and then define the corresponding errors. We start by defining the best possible  approximation of the function $T^{\pi_{\lambda_{k}}}Q_{k,j-1}$ possible from the class of neural networks, with respect to the square loss function with the target being $T^{\pi_{\lambda_{k}}}Q_{k,j-1}$.

\begin{definition} \label{def_1}
   For iteration $k$ of  the outer for loop  and iteration $j$ of the first inner for loop of Algorithm \ref{algo_1}, we define
   \begin{equation}
   Q^{1}_{k,j}=\argminA_{Q_{\theta},\theta \in \Theta^{'}}\mathbb{E}(Q_{\theta}(s,a) - T^{\pi_{\lambda_{k}}}Q_{k,j-1}(s,a))^{2},   \nonumber 
   \end{equation}
where $(s,a) \sim \zeta_{\nu}^{\pi_{\lambda_{k}}}$.
\end{definition}

Note that we do not have access to the transition probability kernel $P$, hence we do not know $ T^{\pi_{\lambda_{k}}}$. To alleviate this, we use the observed next state and actions instead. Using this, we define  $Q^{2}_{k,j}$ as, 

\begin{definition} \label{def_2}
     For iteration $k$ of  the outer for loop  and iteration $j$ of the first inner for loop of Algorithm \ref{algo_1}, we define
   \begin{eqnarray}
    Q^{2}_{k,j} &=& \argminA_{Q_{\theta},\theta \in \Theta^{'}}\mathbb{E}(Q_{\theta}(s,a) \nonumber\\
                &-& (r(s,a)+{\gamma}Q_{k,j-1}(s',a'))^{2},\label{temp} \nonumber
   \end{eqnarray}
\end{definition}

Here the expectation is with respect to ${(s,a) \sim \zeta_{\nu}^{\pi_{\lambda_{k}}},s' \sim P({\cdot}|s,a)}$ and $a^{'} \sim \pi^{\lambda_{k}}(.|s^{'})$. To obtain $Q^{2}_{k,j}$, we still need to compute the true expected value. However, we still do not know the transition function $P$. \ignore{\bf what sampling? sampling distribution?}We thus sample transitions obtained from following the policy $\pi^{\lambda_{k}}$ and minimize the corresponding empirical loss function as follows. Consider the set of $n$ state-action pairs sampled by starting from a state action distribution $\nu$ and following policy $\pi^{\lambda_{k}}$, using which we define $Q^{3}_{k,j}$ as,

\begin{definition} \label{def_3}
   For the set of $n$ state action pairs sampled in iteration $k$ of the outer for loop of Algorithm \ref{algo_1} and at iteration $j$ of the first inner for loop we define
   \begin{eqnarray}
    &&Q^{3}_{k,j} = \argminA_{Q_{\theta},\theta \in \Theta^{'}}\frac{1}{n} \sum_{i=1}^{n}\Big( Q_{\theta}(s_{i},a_{i})  \nonumber\\
    && \  \  - \big(r_{i} + {\gamma}Q_{k,j-1}(s^{'}_{i},a^{'}_{i}) \big)\Big)^{2} \nonumber
   \end{eqnarray}
\end{definition}

$Q^{3}_{k,j}$ is the best possible approximation for $Q$-value function which minimizes the sample average of the square loss functions with the target values as $ \big(r_{i}+{\gamma}Q_{k,j-1}(s^{'}_{i},a^{'}_{i}) \big)$. In other words, this is the optimal solution for fitting the observed data.

We now defined the errors using the $Q$ functions just defined. We start by defining the approximation error which represents the difference between the function $T^{\pi_{\lambda_{k}}}Q_{k,j-1}$ and its best approximation possible from the class of neural networks used for critic parametrization denoted by $Q^{1}_{k,j}$.  
\begin{definition}[Approximation Error] \label{def_4}
    For a given iteration $k$ of the outer for loop and iteration $j$ of the first inner for loop of Algorithm \ref{algo_1}, we define, $\epsilon^{1}_{k,j} =T^{\pi_{\lambda_{k}}}Q_{k,j-1} - Q^{1}_{k,j}$.
\end{definition}

This error is a measure of the approximation power of the class of neural networks we use to represent the critic. We upper bound this error in Lemma \ref{lem_1} in Appendix \ref{supp_lemm_1}. 
\ignore{\bf Only Definition 4 is bounded in Appendix - none before was?}

We also define Estimation Error which denotes the error between the best approximation of $T^{\pi_{\lambda_{k}}}Q_{k,j-1}$  possible from the class of neural networks denoted by $Q^{1}_{k,j}$ and the minimizer of the loss function in Definition \ref{def_2} denoted by $Q^{2}_{k,j}$. 
\begin{definition}[Estimation Error] \label{def_5}
   For a given iteration $k$ of the outer for loop and iteration $j$ of the first inner for loop of Algorithm \ref{algo_1}, we define, $\epsilon^{2}_{k,j} =Q^{1}_{k,j} - Q^{2}_{k,j}$.
\end{definition}

We demonstrate that this error is zero in Lemma \ref{lem_2} in Appendix \ref{supp_lemm_1}.

We now define the Sampling error, which denotes the difference between the minimizer of expected loss function in Definition \eqref{def_2} denoted by $Q^{2}_{k,j}$ and the minimizer of the empirical loss function in Definition \eqref{def_3} denoted by $Q^{3}_{k,j}$. We can see that intuitively, the more samples we have the closer these two functions will be. We use Rademacher complexity results to upper bound this error. Thus this error is a function of the number of samples of transitions collected. We account for the Markov dependence of the transitions in this error. 
\begin{definition}[Sampling Error] \label{def_6}
   For a given iteration $k$ of the outer for loop and iteration $j$ of the first inner for loop of Algorithm \ref{algo_1}, we define, $\epsilon^{3}_{k,j} =Q^{3}_{k,j} - Q^{2}_{k,j}$. 
\end{definition}
An upper bound on this error is established in Lemma \ref{lem_3} in Appendix \ref{supp_lemm_1}.

Finally, we define optimization error which denotes the difference between the minimizer of the empirical square loss function in Definition \eqref{def_3} denoted by $Q^{3}_{k,j}$, and our estimate of this minimizer that is obtained from the gradient descent algorithm that is implemented in lines $5-11$ of Algorithm \eqref{algo_1}.
\begin{definition}[Optimization Error] \label{def_7}
   For a given iteration $k$ of the outer for loop and iteration $j$ of the first inner for loop of Algorithm \ref{algo_1}, we define, $\epsilon^{4}_{k,j} =Q^{3}_{k,j}- Q_{k,j}$.
\end{definition}

The key insight we use to bound this error is the fact that the loss function in Definition \eqref{def_3} can be treated as an expected loss function, with a weight of $\frac{1}{n}$ over all $n$ transition samples. We thus bound this error using tools established in \citet{fu2020single} in Lemma \ref{lem_4} in Appendix \ref{supp_lemm_1}.

\section{Conclusions} \label{Conclusion and Future Work}
In this paper, we study an actor critic algorithm with a neural network used to represent both the the critic and find the sample complexity guarantees for the algorithm. We show that our approach  achieves a last iterate convergence sample complexity of $\tilde{\mathcal{O}}(\epsilon^{-3})$. We do so without assuming i.i.d. sampling and without the restriction of the finite action space. To our knowledge this is the first work that achieves such a result.
\section*{Impact Statement}
This paper presents work whose goal is to advance the field of Machine Learning. There are no potential societal consequences of our work.

\pagebreak
\bibliography{mybib}
\bibliographystyle{icml2024}
\onecolumn
\appendix

\tableofcontents
\newpage
\appendix
\section*{\centering {Appendix}}
\section{Supplementary Lemmas}\label{sup_lem}

Here we provide some definitions and results that will be used to prove the lemmas stated in the  paper.

\begin{definition} \label{def_8}
   For a given set $ Z \subset \mathbb{R}^{n}$, we define the Rademacher complexity of the set $Z$ as 
   \begin{equation}
   Rad(Z) = \mathbb{E} \left(\sup_{z \in Z} \frac{1}{n} \sum_{i=1}^{d}\Omega_{i}z_{i}\right)    
   \end{equation}
   where $\Omega_{i}$ is random variable such that $P(\Omega_{i}=1)=\frac{1}{2}$,  $P(\Omega_{i}=-1)=\frac{1}{2}$ and $z_{i}$ are the co-ordinates of $z$ which is an element of the set $Z$
\end{definition}

\begin{lemma} \label{sup_lem_0}
Consider a set of observed data denoted by $ z = \{z_{1},z_{2},\cdots\,z_{n}\} \in Z$, a parameter space  $\Theta$, a loss function $\{l:Z \times \Theta \rightarrow \mathbb{R}\}$ where  $0 \le l(\theta,z) \le 1$  $\forall (\theta,z) \in \Theta \times Z$. The empirical risk for a set of observed data as $R(\theta)=\frac{1}{n} \sum_{i=1}^{n}l(\theta,z_{i})$ and the population risk  as $r(\theta)= \mathbb{E}l(\theta,\tilde{z_{i}})$, where $\tilde{z_{i}}$ sampled from some distribution over $Z$.

We define a set of functions denoted by $\mathcal{L}$ as 

\begin{equation}
    \mathcal{L}=\{z \in Z \rightarrow l(\theta,z) \in \mathbb{R}:\theta \in \Theta \}
\end{equation}

Given $z=\{z_{1},z_{2},z_{3}\cdots,z_{n}\}$ we further define a set $\mathcal{L} \circ z$ as 

\begin{equation}
    \mathcal{L} \circ z \ =\{ (l(\theta,z_{1}),l(\theta,z_{2}),\cdots,l(\theta,z_{n})) \in \mathbb{R}^{n} : \theta \in \Theta\}
\end{equation}

Then, we have 

\begin{equation}
    \mathbb{E}\sup_{\theta \in \Theta} |\{r(\theta)-R(\theta)\}| \le 2\mathbb{E} \left(Rad(\mathcal{L} \circ z)\right)
\end{equation}

If the data is of the form $z_{i}=(x_{i},y_{i}), x \in X, y \in Y$ and the loss function is of the form $l(a_{\theta}(x),y)$, is $L$ lipschitz and $a_{\theta}:\Theta{\times}X \rightarrow \mathbb{R}$, then we have 

\begin{equation}
    \mathbb{E}\sup_{\theta \in \Theta} |\{r(\theta)-R(\theta)\}| \le 2{L}\mathbb{E} \left(Rad(\mathcal{A} \circ \{x_{1},x_{2},x_{3},\cdots,x_{n}\})\right)
\end{equation}

where \begin{equation}
    \mathcal{A} \circ \{x_{1},x_{2},\cdots,x_{n}\}\ =\{ (a(\theta,x_{1}),a(\theta,x_{2}),\cdots,a(\theta,x_{n})) \in \mathbb{R}^{n} : \theta \in \Theta\}
\end{equation}

\end{lemma}

The detailed proof of the above statement is given in (Rebeschini,  2022)\footnote{ Algorithmic Foundations of Learning [Lecture Notes].https://www.stats.ox.ac.uk/$\sim$rebeschi/teaching/AFoL/22/}. The upper bound for $ \mathbb{E}\sup_{\theta \in \Theta} (\{r(\theta)-R(\theta)\})$ is proved in the aformentioned reference. However, without loss of generality the same proof holds for the upper bound for $ \mathbb{E}\sup_{\theta \in \Theta} (\{R(\theta)-r(\theta)\})$. Hence the upper bound for $ \mathbb{E}\sup_{\theta \in \Theta}|\{r(\theta)-R(\theta)\}|$ can be established.

\begin{lemma} \label{sup_lem_1}
Consider three random random variable $x \in \mathcal{X} $ and  $y,y^{'} \in \mathcal{Y}$. Let $\mathbb{E}_{x,y}, \mathbb{E}_{x}$ and $\mathbb{E}_{y|x}$, $\mathbb{E}_{y^{'}|x}$  denote the expectation with respect to the joint distribution of $(x,y)$, the marginal distribution of $x$, the conditional distribution of $y$ given $x$ and the conditional distribution of $y^{'}$ given $x$ respectively . Let $f_{\theta}(x)$ denote a bounded measurable function of $x$ parameterised by some parameter $\theta$ and $g(x,y)$ be bounded measurable function of both $x$ and $y$.

Then we have

\begin{equation}
    \argminA_{f_{\theta}}\mathbb{E}_{x,y}\left(f_{\theta}(x)-g(x,y)\right)^{2}=\argminA_{f_{\theta}} \left(\mathbb{E}_{x}\left(f_{\theta}(x)-\mathbb{E}_{y^{'}|x}(g(x,y^{'}))\right)^{2}\right) \label{sup_lem_1_1}
\end{equation}    
\end{lemma}

\begin{proof}
Denote the left hand side of Equation \eqref{sup_lem_1_1} as $\mathbb{X}_{\theta}$, then add and subtract $\mathbb{E}_{y|x}(g(x,y^{'}))$ to it to get 
\begin{align}
     \mathbb{X}_{\theta}=& \argminA_{f_{\theta}}\left(\mathbb{E}_{x,y}\left(f_{\theta}(x)-\mathbb{E}_{y^{'}|x}(g(x,y^{'}))+\mathbb{E}_{y^{'}|x}(g(x,y^{'}))-g(x,y)\right)^{2}\right) \label{sup_lem_1_2}
     \\
     =&  \argminA_{f_{\theta}}\Big(\mathbb{E}_{x,y}\left(f_{\theta}(x)-\mathbb{E}_{y^{'}|x}(g(x,y^{'}))\right)^{2} + \mathbb{E}_{x,y}\left(g(x,y)-\mathbb{E}_{y^{'}|x}(g(x,y^{'}))\right)^{2} 
     \nonumber
     \\ 
     &\qquad-2\mathbb{E}_{x,y}\Big(f_{\theta}(x)-\mathbb{E}_{y^{'}|x}(g(x,y^{'}))\Big)\left(g(x,y)-\mathbb{E}_{y^{'}|x}(g(x,y^{'}))\right)\Big) .\label{sup_lem_1_3}
\end{align}
Consider the third term on the right hand side of Equation \eqref{sup_lem_1_3}
\begin{align}
    2\mathbb{E}_{x,y}&\left(f_{\theta}(x)-\mathbb{E}_{y^{'}|x}(g(x,y^{'}))\right)\left(g(x,y)-  \mathbb{E}_{y^{'}|x}(g(x,y^{'}))\right) 
    \nonumber
    \\
    =& 2\mathbb{E}_{x}\mathbb{E}_{y|x}\left(f_{\theta}(x)-\mathbb{E}_{y^{'}|x} (g(x,y^{'}))\right)
  \left(g(x,y)-\mathbb{E}_{y^{'}|x}(g(x,y^{'}))\right)\label{sup_lem_1_4}
    \\
    =& 2\mathbb{E}_{x}\left(f_{\theta}(x)-\mathbb{E}_{y^{'}|x}(g(x,y^{'}))\right)\mathbb{E}_{y|x}\left(g(x,y)-\mathbb{E}_{y^{'}|x}(g(x,y^{'}))\right) \label{sup_lem_1_5}
    \\
    =& 2\mathbb{E}_{x}\left(f_{\theta}(x)-\mathbb{E}_{y^{'}|x}(g(x,y^{'}))\right)\left(\mathbb{E}_{y|x}(g(x,y))-\mathbb{E}_{y|x}\left(\mathbb{E}_{y^{'}|x}(g(x,y^{'}))\right)\right)
    \label{sup_lem_1_6}
    \\
    =& 2\mathbb{E}_{x}\left(f_{\theta}(x)-\mathbb{E}_{y^{'}|x}(g(x,y^{'}))\right)\Big(\mathbb{E}_{y|x}(g(x,y))  -\mathbb{E}_{y^{'}|x}(g(x,y^{'}))\Big)  \label{sup_lem_1_7}\\
    =& 0
\end{align}

Equation \eqref{sup_lem_1_4} is obtained by writing $\mathbb{E}_{x,y}=\mathbb{E}_{x}\mathbb{E}_{y|x}$ from the law of total expectation. Equation \eqref{sup_lem_1_5} is obtained from  \eqref{sup_lem_1_4} as the term $f_{\theta}(x)-\mathbb{E}_{y^{'}|x}(g(x,y^{'}))$ is not a function of $y$. Equation \eqref{sup_lem_1_6} is obtained from \eqref{sup_lem_1_5} as $\mathbb{E}_{y|x}\left(\mathbb{E}_{y^{'}|x}(g(x,y^{'}))\right)=\mathbb{E}_{y^{'}|x}(g(x,y^{'}))$ because $\mathbb{E}_{y^{'}|x}(g(x,y^{'}))$ is not a function of $y$ hence is constant with respect to the expectation operator $\mathbb{E}_{y|x}$. 

Thus plugging in value of $  2\mathbb{E}_{x,y}\left(f_{\theta}(x)-\mathbb{E}_{y^{'}|x}(g(x,y^{'}))\right)\left(g(x,y)-  \mathbb{E}_{y^{'}|x}(g(x,y^{'}))\right)$ in Equation \eqref{sup_lem_1_3} we get 
\begin{align}    
\argminA_{f_{\theta}}\mathbb{E}_{x,y}\left(f_{\theta}(x)-g(x,y)\right)^{2} =  & \argminA_{f_{\theta}} (\mathbb{E}_{x,y}\left(f_{\theta}(x)-\mathbb{E}_{x,y^{'}}(g(x,y^{'}))\right)^{2} 
\nonumber\\
&+ \mathbb{E}_{x,y}\left(g(x,y)-\mathbb{E}_{y^{'}|x}(g(x,y^{'}))\right)^{2}). \label{sup_lem_1_8}
\end{align}
Note that the second term on the right hand side of Equation \eqref{sup_lem_1_8} des not depend on $f_{\theta}(x)$ therefore we can write Equation \eqref{sup_lem_1_8} as 

\begin{equation}
    \argminA_{f_{\theta}}\mathbb{E}_{x,y}\left(f_{\theta}(x)-g(x,y)\right)^{2} =  \argminA_{f_{\theta}} \left(\mathbb{E}_{x,y}\left(f_{\theta}(x)-\mathbb{E}_{y^{'}|x}(g(x,y^{'}))\right)^{2}\right) \label{sup_lem_1_9}
\end{equation}

Since the right hand side of Equation \eqref{sup_lem_1_9} is not a function of $y$ we can replace $\mathbb{E}_{x,y}$ with $\mathbb{E}_{x}$ to get 

\begin{equation}
    \argminA_{f_{\theta}}\mathbb{E}_{x,y}\left(f_{\theta}(x)-g(x,y)\right)^{2} =  \argminA_{f_{\theta}} \left(\mathbb{E}_{x}\left(f_{\theta}(x)-\mathbb{E}_{y^{'}|x}(g(x,y^{'}))\right)^{2}\right) \label{sup_lem_1_10}
\end{equation}
\end{proof}

\section{Supporting Lemmas} \label{supp_lemm_1}
We will now state the key lemmas that will be used for finding the sample complexity of the proposed algorithm. 

\begin{lemma} \label{lem_1}
    For a given iteration $k$ of  the outer for loop  and iteration $j$ of the first inner for loop of Algorithm \ref{algo_1}, the approximation error denoted by $\epsilon^{1}_{k,j}$ in Definition \ref{def_4}, we have 
    \begin{equation}
        \mathbb{E}\left(|\epsilon^{1}_{k,j}|\right) \le \sqrt{\epsilon_{approx}},
    \end{equation}
\end{lemma}

Where the expectation is with respect to and $(s,a) \sim \zeta_{\nu}^{\pi_{\lambda_{k}}}(s,a)$

\textit{Proof Sketch:} We use Assumption \ref{assump_7} and the definition of the variance of a random variable to obtain the required result. The detailed proof is given in Appendix \ref{proof_lem_1}. 


\begin{lemma} \label{lem_2}
     For a given iteration $k$ of  the outer for loop  and iteration $j$ of the first inner for loop of Algorithm \ref{algo_1},  $Q^{1}_{k,j}=Q^{2}_{k,j}$, or equivalently $\epsilon^{2}_{k,j}=0$
\end{lemma}

\textit{Proof Sketch:} We use  Lemma \ref{sup_lem_1} in Appendix \ref{sup_lem} and use the definitions of $Q^{1}_{k,j}$ and $Q^{2}_{k,j}$ to prove this result. The detailed proof is given in Appendix \ref{proof_lem_2}.

\begin{lemma} \label{lem_3}
    For a given iteration $k$ of  the outer for loop  and iteration $j$ of the first inner for loop of Algorithm \ref{algo_1}, if the number of state action pairs sampled are denoted by $n$, then the error $\epsilon^{3}_{k,j}$ defined in Definition \ref{def_6} is upper bounded as
    \begin{equation}
\mathbb{E}\left(|\epsilon^{3}_{k,j}|\right)\le  \tilde{\mathcal{O}}\left(\frac{1}{\sqrt{n}}\right), 
    \end{equation}
\end{lemma}
Where the expectation is with respect to and $(s,a) \sim \zeta_{\nu}^{\pi_{\lambda_{k}}}(s,a)$

\textit{Proof Sketch:} First we note that For a given iteration $k$ of Algorithm \ref{algo_1} and iteration $j$ of the first for loop of Algorithm \ref{algo_1},  $\mathbb{E}(R_{X_{k,j},Q_{k,j-1}}({\theta})) = L_{Q_{j,k-1}}({\theta})$ where $R_{X_{k,j},Q_{j,k-1}}({\theta})$ and $L_{Q_{j,k-1}}({\theta})$ are defined in Appendix \ref{proof_lem_3}. We use this to get a probabilistic bound on the expected value of $|(Q^{2}_{j,k}) - (Q^{3}_{j,k})|$ using Rademacher complexity theory when the samples are drawn from a Markov chain satisfying Asssumption \ref{assump_4}. The detailed proof is given in Appendix \ref{proof_lem_3}.

\begin{lemma} \label{lem_4}
      For a given iteration $k$ of  the outer for loop  and iteration $j$ of the first inner for loop of Algorithm \ref{algo_1}, let the number of gradient descent steps be denoted by $L$, and the gradient descent step size $\beta^{'}$ satisfy
    \begin{eqnarray}
      \beta^{'} = \frac{1}{\sqrt{L}},
    \end{eqnarray}
   Then with probability at least $1-\Omega\left(exp-\left(m^{\frac{2}{3}}D\right)\right)  $ the error $\epsilon_{k_{4}}$ defined in Definition \ref{def_7} is upper bounded as
    \begin{equation}
       \mathbb{E}(|\epsilon^{4}_{k,j}|) \le {\mathcal{O}}\left(L^{-\frac{1}{4}}\right) + \mathcal{O}\left(m^{-\frac{1}{12}}D^{\frac{7}{2}}\right) ,
    \end{equation}
\end{lemma}
Where the expectation is with respect to $(s,a) \sim \zeta_{\nu}^{\pi_{\lambda_{k}}}(s,a)$.

\textit{Proof Sketch:} They key insight we use here is that the loss function in Definition \ref{def_3} can be considered as an expectation over the state action transitions samples at iteration $k$ with an equal probability of selecting each transitions. This is then combined with results from \citet{fu2020single} to obtain the desired result.

\section{Proof of Theorem \ref{thm}} \label{thm proof} 
\begin{proof}
From the smoothness property of the expected return we have 
\begin{eqnarray}
    -J(\lambda_{t+1}) &\le& -J(\lambda_{t}) - \langle {\nabla}J(\lambda_{t}),\lambda_{t+1} - \lambda_{t} \rangle + {L_{J}}||\lambda_{t+1} - \lambda_{t}||^{2}  \label{13_1}\\
    &\le& -J(\lambda_{t}) - {{\alpha}_{t}}\frac{\langle {\nabla}J(\lambda_{t}),d_{t} \rangle}{||d_{t}||} + {L_{J}}||\lambda_{t+1} - \lambda_{t}||^{2} \label{13_3}
\end{eqnarray}

Here $\alpha_{t}$ is the actor step size at iteration $t$ of Algorithm \ref{algo_1}. Now define the term  $e_{t} =  d_{t}-{\nabla}J(\lambda_{t})$. 

Consider two cases, first if $||e_{t}|| \le \frac{1}{2}||{\nabla}J(\lambda_{t})||$, then we have 

\begin{eqnarray}
                   -\frac{\langle {\nabla}J(\lambda_{t}),d_{t} \rangle}{||d_{t}||} &=& \frac{-||{\nabla}J(\lambda_{t})||^{2} -  \langle {\nabla}J(\lambda_{t}),e_{t} \rangle}{||d_{t}||} \\
                   &\le& \frac{-||{\nabla}J(\lambda_{t})||^{2} +   ||{\nabla}J(\lambda_{t})||{\cdot}||e_{t}||}{||d_{t}||} \\
                   &\le& \frac{-||{\nabla}J(\lambda_{t})||^{2} +   ||{\nabla}J(\lambda_{t})||{\cdot}||e_{t}||}{||d_{t}||} \\
                   &\le& \frac{-||{\nabla}J(\lambda_{t})||^{2} +   \frac{1}{2}||{\nabla}J(\lambda_{t})||^{2}}{||d_{t}||} \\
                   &\le& -\frac{||{\nabla}J(\lambda_{t})||^{2}}{2(||e_{t}|| + ||{\nabla}J(\lambda_{t})||)} \\
                   &\le& -\frac{1}{3}||{{\nabla}J(\lambda_{t})}||
\end{eqnarray}

If $||e_{t}|| \ge \frac{1}{2}||{\nabla}J(\lambda_{t})||$, then we have 

\begin{eqnarray}
                   \frac{\langle {\nabla}J(\lambda_{t}),d_{t} \rangle}{||d_{t}||} &\le& ||{\nabla}J(\lambda_{t})||    \\
                   &=& -\frac{1}{3}||{\nabla}J(\lambda_{t})|| + \frac{4}{3}||{\nabla}J(\lambda_{t})|| \\
                   &\le& -\frac{1}{3}||{\nabla}J(\lambda_{t})|| + \frac{8}{3}||e_{t}|| \label{fat_1}
\end{eqnarray}

This inequality has been established in \citet{fatkhullin2023stochastic}.  Now using Equation \eqref{fat_1} in  Equation \eqref{13_3} we get

\begin{eqnarray}
    -J(\lambda_{t+1}) &\le&  -J(\lambda_{t}) -  \frac{{{\alpha}_{t}}}{3}||{\nabla}J(\lambda_{t})|| +  \frac{8{{\alpha}_{t}}}{3}||e_{t}||  + {L_{J}}||\lambda_{t+1} - \lambda_{t}||^{2}  \label{13_4}
\end{eqnarray}

Now from Lemma \ref{lem_1} we have 
 
\begin{eqnarray}
    -J(\lambda_{t+1}) &\le&  -J(\lambda_{t}) -  \frac{{{\alpha}_{t}}{\sqrt{\mu^{'}}}}{3}(J^{*} -J(\lambda_{t})) +  \frac{8{\alpha}_{t}}{3}||{\nabla}J(\lambda_{t}) - d_{t}|| \nonumber\\
     &+& {L_{J}}||\lambda_{t+1} - \lambda_{t}||^{2} + \frac{{\alpha}_{t}}{3}\epsilon^{'} \label{13_6}\\
    J^{*} -J(\lambda_{t+1}) &\le&  J^{*} -J(\lambda_{t}) -  \frac{{{\alpha}_{t}}{\sqrt{\mu^{'}}}}{3}(J^{*} -J(\lambda_{t}))  + \frac{8{\alpha}_{t}}{3}||{\nabla}J(\lambda_{t}) - d_{t} || \nonumber\\ 
    &+& {L_{J}}||\lambda_{t+1} - \lambda_{t}||^{2} + \frac{{\alpha}_{t}}{3}\epsilon^{'} \label{13_7}\\
    \delta_{\lambda_{t+1}} &\le&  \left( 1- \frac{{{\alpha}_{t}}{\sqrt{\mu^{'}}}}{3}\right)\delta_{\lambda_{t}}  
     + \frac{8{\alpha}_{t}}{3}||{\nabla}J(\lambda_{t}) - d_{t} ||   \nonumber\\
    &+& {L_{J}}||\lambda_{t+1} - \lambda_{t}||^{2} + \frac{{\alpha}_{t}}{3}\epsilon^{'} \label{13_8} \\
    &\le&  \left( 1- \frac{{{\alpha}_{t}}{\sqrt{\mu^{'}}}}{3}\right)\delta_{\lambda_{t}}  
     + \frac{8{\alpha}_{t}}{3}||{\nabla}J(\lambda_{t}) - d_{t} ||   \nonumber\\
    &+& {L_{J}}{{\alpha}_{t}}^{2} + \frac{{\alpha}_{t}}{3}\epsilon^{'}  \label{13_9}
\end{eqnarray}

Where $\delta_{t} = J^{*} - J(\lambda_{t}) $Now, in equation \eqref{13_9}, if we plug in the value of $\delta_{\lambda_{t}}$ by evaluating equation \eqref{13_9} for $t-1$, we get the following

\begin{eqnarray}
  \delta_{\lambda_{t+1}}  &\le&  \left( 1- \frac{{\alpha}_{t}{\cdot}\sqrt{\mu^{'}}}{3}\right)\left( 1- \frac{{\alpha}_{t-1}{\cdot}\sqrt{\mu^{'}}}{3}\right)\delta_{\lambda_{t-1}}  \nonumber\\    
     &+&   \left(1-\frac{{\alpha}_{t}{\sqrt{\mu^{'}}}}{3}\right){{\alpha}_{t-1}}(||{\nabla}J(\lambda_{t-1}) - d_{t-1}|| + \epsilon^{'}) + {{\alpha}_{t}}(||{\nabla}J(\lambda_{t}) - d_{t}|| + \epsilon^{'}) \nonumber\\
     &+&  \left(1-\frac{{\alpha}_{t}{\sqrt{\mu^{'}}}}{3}\right){L_{J}}{{\alpha}_{t-1}}^{2} +  {L_{J}}{{\alpha}_{t}}^{2} \label{13_10}
\end{eqnarray}

We got rid of the  factor of $\frac{8}{3}$ and $\frac{1}{3}$ on $||{\nabla}J(\lambda_{t}) - d_{t} ||$ and $\epsilon^{'}$ respectively as it will be absorbed in the $\tilde{\mathcal{O}}$ notation. If we repeat the substitution starting from $t$ and  going back till $t=2$ we get 

\begin{eqnarray}
  \delta_{\lambda_{t}}  &\le&   \underbrace{\Pi^{k=t}_{k=2}\left(1-\frac{{\alpha}_{k}{\sqrt{\mu^{'}}}}{3}\right)\delta_{\lambda_{2}}}_{A}   \nonumber\\    
     &+&  \underbrace{\sum^{k=t-2}_{k=0}  \left(\Pi_{i=0}^{k-1} \left(1-\frac{{\alpha}_{(t-i)}{\sqrt{\mu^{'}}}}{3}\right)\right)^{\mathbb{1}(k \ge 1)}\alpha_{t-k}(||{\nabla}J(\lambda_{t-k}) - d_{t-k}|| + \epsilon^{'})}_{B} \nonumber\\
     &+&  \underbrace{{L_{J}}\sum^{k=t-2}_{k=0} \left(\Pi_{i=0}^{i=k-1} \left(1-\frac{{\alpha}_{(t-i)}{\sqrt{\mu^{'}}}}{3}\right)\right)^{\mathbb{1}(k \ge 1)}(\alpha_{t-k})^{2}}_{C}  \label{13_11}
\end{eqnarray}

Now let us consider the term $A$ is equation \eqref{13_11}, if ${\alpha}_{k} = \frac{\alpha}{k}$ where $\alpha = \frac{7}{2{\sqrt{\mu^{'}}}}$, then we have 
\begin{eqnarray}
    1 - \frac{{{\alpha}_{k}}\sqrt{\mu^{'}}}{3}  &=& 1- \frac{7}{6k}  \\
                            &\le& 1- \frac{1}{k}  \\
                            &\le& \frac{k-1}{k} \\
                            &\le&  \frac{{\alpha}_{k}}{{\alpha}_{k-1}}
\end{eqnarray}

Therefore we have 

\begin{eqnarray}
    A= \Pi^{k=t}_{k=2}\left(1 - \frac{{{\alpha}_{k}}\sqrt{\mu^{'}}}{3}\right)\delta_{\lambda_{2}}  &\le&  \Pi^{k=t}_{k=2} \left(\frac{{\alpha}_{k}}{{\alpha}_{k-1}}\right)\delta_{\lambda_{2}} \\
                                              &\le&  \frac{{\alpha}_{t}}{\alpha_{1}}\delta_{\lambda_{2}} = \frac{1}{t}\delta_{\lambda_{2}} \label{13_11_1}
\end{eqnarray}

Now let us consider the term $B$ is equation \eqref{13_11}
\begin{eqnarray}
    B &=& \sum^{k=t-2}_{k=0}  \left(\Pi_{i=0}^{k-1} \left(1-\frac{{\alpha}_{(t-i)}{\sqrt{\mu^{'}}}}{3}\right)\right)^{\mathbb{1}(k \ge 1)}\alpha_{t-k}(||{\nabla}J(\lambda_{t-k}) - d_{t-k}|| + \epsilon^{'}) \label{13_12}
\end{eqnarray}

Now if we consider the coefficients of $(||{\nabla}J(\lambda_{t-k}) -d_{t-k}|| + \epsilon^{'})$, we see as follows.

For $k=0$ the product term is $1$ because of the indicator function $\mathbb{1}(k \ge 1)$.

For $k=1$, suppose the coefficient is $\alpha_{k} = \frac{\alpha}{k}$. Then we have

\begin{eqnarray}
    \left(1- \frac{{\alpha}{\sqrt{\mu^{'}}}}{3t}\right)\frac{\alpha}{t-1} =   \left(\frac{t-\frac{{\alpha}{\sqrt{\mu^{'}}}}{3}}{t-1}\right)\frac{\alpha}{t} \label{13_13}
\end{eqnarray}

For $k=2$ we have

\begin{eqnarray}
    \left(1- \frac{\frac{{\alpha}{\sqrt{\mu^{'}}}}{3}}{t}\right) \left(1- \frac{\frac{{\alpha}{\sqrt{\mu^{'}}}}{3}}{t-1}\right)\frac{\alpha}{t-2} =   \left(\frac{t-\frac{{\alpha}{\sqrt{\mu^{'}}}}{3}-1}{t-2}\right)\left(\frac{t-\frac{{\alpha}{\sqrt{\mu^{'}}}}{3}}{t-1}\right)\frac{\alpha}{t} \label{13_14}
\end{eqnarray}

In general for a $k$ this coefficient is thus 

\begin{eqnarray}
        {\Pi_{i=1}^{k}} \left(\frac{t- (\frac{{\alpha}{\sqrt{\mu^{'}}}}{3} + i -1)}{t-i}\right)\frac{\alpha}{t} \label{13_15}
\end{eqnarray}

Now for ${\alpha} = \frac{7}{2\sqrt{\mu}}$  the numerator in all the product terms is less than the denominator hence product term is less than 1. Therefore, all the coefficients in $B$ are upper bounded by ${\alpha}_{t}$. Thus we have

\begin{eqnarray}
         B &\le&   \frac{\alpha}{t}\sum^{k=t-2}_{k=1}(||{\nabla}J(\lambda_{k}) - d_{k}||) + \epsilon^{'} \label{13_16}
\end{eqnarray}

Now let us consider the term $C$ is equation \eqref{13_11}
\begin{eqnarray}
    C &=&  {L_{J}}\sum^{k=t-2}_{k=0} \left(\Pi_{i=0}^{i=k-1} \left(1-\frac{{\alpha}_{(t-i)}{\sqrt{\mu^{'}}}}{3}\right)\right)^{\mathbb{1}(k \ge 1)}(\alpha_{t-k})^{2} \label{13_17}
\end{eqnarray}

Now similar to what was done for $A$ consider the coefficients of ${\alpha_{t-k}}^{2}$. 

For $k=0$ the product term is $1$ because of the indicator function $\mathbb{1}(k \ge 1)$.

for $k=1$ if we have ${\alpha} = \frac{7}{2\sqrt{\mu}}$ then

\begin{eqnarray}
    \left(1- \frac{{\alpha}{\sqrt{\mu^{'}}}}{3t}\right)\left(\frac{\alpha}{t-1}\right)^{2}  &\le&  \left(\frac{\alpha}{t-1}\right)^{2} \label{13_18}
\end{eqnarray}

for $k=2$ if we have ${\alpha} = \frac{7}{2\sqrt{\mu}}$ then

\begin{eqnarray}
    \left(1-\frac{{\alpha}{\mu^{'}}}{t}\right)\left(1-\frac{{\alpha}{\mu^{'}}}{t-1}\right)\left(\frac{\alpha}{t-2}\right)^{2} &=& \frac{\left(t- \frac{{\alpha}{\sqrt{\mu^{'}}}}{3}\right)}{t} \frac{\left(t- \frac{{\alpha}{\sqrt{\mu^{'}}}}{3}-1\right)}{t-1}\left(\frac{\alpha}{t-2}\right)^{2} \label{13_19}  \\
    &\le&  \left(\frac{\alpha}{t-2}\right)^{2} 
\end{eqnarray}
This is because both terms in the coefficient of $\left(\frac{\alpha}{t-2}\right)^{2}$ are less than 1.

In general for any $k $ if we have ${\alpha} = \frac{7}{2\sqrt{\mu}}$ then we have 
\begin{eqnarray}
    \Pi^{i=k-1}_{i=0}\left(1-\frac{{\alpha}{\mu^{'}}}{t-i}\right)\left(\frac{\alpha}{t-k}\right)^{2} &=& \frac{\left(t- \frac{{\alpha}{\sqrt{\mu^{'}}}}{3}\right)}{t} \frac{\left(t-\frac{{\alpha}{\sqrt{\mu^{'}}}}{3} - 1\right)}{t-1}{\cdots} \frac{\left(t- \frac{{\alpha}{\sqrt{\mu^{'}}}}{3}-k+1\right)}{t-k+1}\left(\frac{\alpha}{t-k}\right)^{2}  \\
    &\le&  \left(\frac{\alpha}{t-k}\right)^{2}   \label{13_20}
\end{eqnarray}
Therefore we have 

\begin{eqnarray}
         C &\le&   L_{J}\sum^{k=t-2}_{k=2} \left(\frac{\alpha}{t-k}\right)^{2} \label{13_20_1}\\
           &\le&   \frac{L_{J}{\cdot}{\alpha^{2}}}{t} \label{13_21} 
\end{eqnarray} 
We get Equation \eqref{13_21} from \eqref{13_20_1} by using the fact that $\sum_{k=1}^{t}\frac{1}{k^{2}} \le \frac{1}{t}$. Now plugging equation \eqref{13_11_1}, \eqref{13_16} and \eqref{13_21} into equation \eqref{13_11} we get

\begin{eqnarray}
  \delta_{\lambda_{t}}  &\le&   \left(\frac{1}{t}\right)\delta_{\lambda_{2}}  + \frac{\alpha}{t}\underbrace{\sum^{k=t-2}_{k=0}(||{\nabla}J(\lambda_{k}) - d_{t})||)}_{A^{'}} 
     +   \frac{L_{J}{\cdot}{\alpha^{2}}}{t}  + \epsilon^{'} \nonumber
     \\ \label{13_22}
\end{eqnarray}

Now consider the terms inside $A^{'}$, if we define $H_{k,J}  =  \mathbb{E}_{(s,a) \sim d^{\pi_{\lambda_{k}}}_{\nu}}({\nabla}{\log}{\pi_{\lambda_{k}}}(a|s)Q_{k,J}(s,a))$
\begin{eqnarray}
    (||{\nabla}J(\lambda_{k}) - d_{t} + H_{k,J} -H_{k,J}||) &\le& ||{\nabla}J(\lambda_{k}) - H_{k,J} ||)  +  ||d_{t} -  H_{k,J}||  \label{13_22_1} \\
    &\le&  ||{\nabla}J(\lambda_{k}) -  H_{k,J}|| +\Bigg|\Bigg|\frac{1}{n}\sum_{i=1}^{n}\left({\nabla}{\log}{\pi_{\lambda_{k}}}(a_{i}|s_{i})Q_{k,J}(s_{i},a_{i}) - (H_{k,J})\right)\Bigg|\Bigg|  \label{13_22_4} \\
    &\le&  ||{\nabla}J(\lambda_{k}) -  H_{k,J}||+ \nonumber\\
    &&  \sqrt{d{\cdot}\sum_{p=1}^{d}\left(\left(\sum_{i=1}^{n}\frac{1}{n}{\nabla}{\log}{\pi_{\lambda_{k}}}(a_{i}|s_{i})Q_{k,J}(s_{i},a_{i})\right)_{p} - (H_{k,J})_{p}\right)^{2}}  \label{13_22_4_1} \\
    \mathbb{E}_{(s,a) \sim d^{\pi_{\lambda_{k}}}_{\nu}}(||{\nabla}J(\lambda_{k}) - d_{t}||) &\le&  ||{\nabla}J(\lambda_{k}) -  H_{k,J}|| + \nonumber\\
    &&  \sqrt{d{\cdot}\sum_{p=1}^{d} \mathbb{E}_{(s,a) \sim d^{\pi_{\lambda_{k}}}_{\nu}}\left(\left(\sum_{i=1}^{n}\frac{1}{n}{\nabla}{\log}{\pi_{\lambda_{k}}}(a_{i}|s_{i})Q_{k,J}(s_{i},a_{i})\right)_{p} - (H_{k,J})_{p}\right)^{2}} \label{13_22_5} \nonumber\\
    && \\
    &\le&  ||{\nabla}J(\lambda_{k}) -  H_{k,J}|| +  \frac{1}{\sqrt{n}}{d}{M_{g}}V_{max} \label{13_22_6} \\
    &\le&  M_{g}\mathbb{E}_{(s,a) \sim d^{\pi_{\lambda_{k}}}_{\nu}}|Q^{\pi_{\lambda_{k}}}(s,a) - Q_{k,J}(s,a)| +   \frac{1}{\sqrt{n}}{d}M_{g}V_{max}  \label{13_22_8}    
\end{eqnarray}

We obtain Equation \eqref{13_22_4_1} from Equation \eqref{13_22_4} by noting that  l1 norm is upper bounded by the l2 norm multiplied by the square root of the dimensions. Here  $({\nabla}{\log}{\pi_{\lambda_{k}}}(a_{i}|s_{i})Q_{k,J}(s_{i},a_{i}))_{p}$  and $(H_{k,J})_{p}$ in Equation \eqref{13_22_4_1} are the $p^{th}$ co-ordinates of the gradients. We obtain Equation \eqref{13_22_5} from  Equation \eqref{13_22_4_1} by taking the expectation with respect to $(s,a) \sim d^{\pi_{\lambda_{k}}}_{\nu}$ on both sides of Equation \eqref{13_22_4_1} and applying Jensen's inequality on the final term on the right hand side. Note that the first term on the right hand side will remain unaffected as it is already an expectation term. We obtain Equation \eqref{13_22_6} from  Equation \eqref{13_22_5} by noting that the variance of the random variable ${\nabla}{\log}{\pi_{\lambda_{k}}}(a|s)Q_{k,J}(s,a)$ is bounded when the expectation is over $(s,a)$. We combine this with the fact that the variance of the mean is the variance divided by the number of samples, which is this case is $n$. We can assume the variance of ${\nabla}{\log}{\pi_{\lambda_{k}}}(a|s)Q_{k,J}(s,a)_{p}$ is bounded because from Assumption \ref{assump_1} we have $||{\nabla}{\log}{\pi_{\lambda_{k}}}(a|s)|| \le M_{g}$ and  from  lemma F.4 of \cite{fu2020single} and the fact that the state action space is bounded we have that $Q_{k,J}$ will be bounded. We obtain Equation \eqref{13_22_8} from Equation \eqref{13_22_6} by using the policy gradient identity which states that ${\nabla}J(\lambda_{k}) = \mathbb{E}_{(s,a) \sim d^{\pi_{\lambda_{k}}}_{\nu}} {\nabla}log{\pi_{\lambda_{k}}}(a|s)Q^{\pi_{\lambda_{k}}}(s,a)$  and by using Assumption \ref{assump_1} which states that $||{\nabla}{\log}{\pi_{\lambda_{k}}}(a|s)|| \le M_{g}$ and 

Now taking the expectation over $(s,a) \sim d^{\pi_{\lambda_{k}}}_{\nu}$ on both sides of Equation \eqref{13_22} and substituting into it the result from Equation \eqref{13_22_8} we get 

\begin{eqnarray}
  \delta_{\lambda_{t}}  &\le&   \left(\frac{1}{t}\right)\delta_{\lambda_{2}}  + \frac{\alpha}{t}M_{g}\sum^{k=t-2}_{k=0}\mathbb{E}_{(s,a) \sim d^{\pi_{\lambda_{k}}}_{\nu}}|Q^{\pi_{\lambda_{k}}}(s,a) - Q_{k,J}(s,a)|
     +  \frac{1}{\sqrt{n}}{d}{M_{g}}V_{max}  +   \frac{L_{J}{\cdot}{\alpha^{2}}}{t}  + \epsilon^{'} \nonumber
     \\ \label{13_22_9}
\end{eqnarray}

We now want to bound the term $\mathbb{E}_{(s,a) \sim d^{\pi_{\lambda_{k}}}_{\nu}}|Q^{\pi_{\lambda_{k}}}(s,a) - Q_{k,J}(s,a)|$.
Let $Q_{k,j}$ denotes our estimate of the action value function at iteration $k$ of Algorithm \ref{algo_1} and iteration $j$ of the first inner for loop of Algorithm \ref{algo_1}.  $Q^{\pi_{\lambda_{k}}}$ denotes the action value function induced by the policy $\pi_{\lambda_{k}}$.
 
    Consider $\epsilon_{k,j+1}= T^{\pi_{\lambda_{k}}}Q_{k,j}-Q_{k,j+1}$.

Thus we get,
    \begin{align}
        Q^{\pi_{\lambda_{k}}}-Q_{k,j+1} =& T^{\pi_{\lambda_{k}}}Q^{\pi_{\lambda_{k}}} -  Q_{k,j+1}  \label{thm_1_1_1}\\
                      =& T^{\pi_{\lambda_{k}}}Q^{\pi_{\lambda_{k}}} -  T^{\pi_{\lambda_{k}}}Q_{k,j} + T^{\pi_{\lambda_{k}}}Q_{k,j} \nonumber\\
                      &  - Q_{k,j+1}  \label{thm_1_1_2}\\
                      =& \gamma(P^{\pi_{\lambda_{k}}}(Q^{\pi_{\lambda_{k}}}-Q_{k,j})) + \epsilon_{k,j+1} \label{thm_1_1_5} \\
        |Q^{\pi_{\lambda_{k}}}-Q_{k,j+1}| \le&      \gamma(P^{\pi_{\lambda_{k}}}(|Q^{\pi_{\lambda_{k}}}-Q_{k,j}|)) + |\epsilon_{k,j+1}|                                                                                                 \label{thm_1_1_6}
    \end{align}

Right hand side of Equation \eqref{thm_1_1_1} is obtained by writing $Q^{\pi_{\lambda_{k}}} = T^{\pi_{\lambda_{k}}}Q^{\pi_{\lambda_{k}}}$. This is because the function $Q^{\pi_{\lambda_{k}}}$ is a stationary point with respect to the operator $T^{\pi_{\lambda_{k}}}$. Equation \eqref{thm_1_1_2} is obtained from \eqref{thm_1_1_1} by adding and subtracting $T^{\pi^{\lambda_{k}}}$. We get \eqref{thm_1_1_6} from \eqref{thm_1_1_5} by taking the absolute value on both sides and applying the triangle inequality on the right hand side.

By recursion on $j$, backwards to $j=0$, we get,

\begin{equation}
    |Q^{\pi_{\lambda_{k}}}-Q_{k,J}| \leq \sum_{j=0}^{J-2} \gamma^{J-j-1} (P^{\pi_{\lambda_{k}}})^{J-j-1}|\epsilon_{k,j+1}| +  \gamma^{J} (P^{\pi_{\lambda_{k}}})^{J}(|Q^{\pi_{\lambda_{k}}}-Q_{0}|) \label{thm_1_2}
\end{equation}

From this we obtain

\begin{align}
    \mathbb{E}_{(s,a) \sim d^{\pi_{\lambda_{k}}}_{\nu}}{|Q^{\pi_{\lambda_{k}}}-Q_{k,J}|}  
    \leq & \sum_{k=0}^{J-2} \gamma^{J-j-1}\mathbb{E}_{(s,a) \sim d^{\pi_{\lambda_{k}}}_{\nu}}((P^{\pi_{\lambda_{k}}})^{K-J-1}|\epsilon_{k,j+1}|) \nonumber
    \\
    &+  \gamma^{J}\mathbb{E}_{(s,a) \sim d^{\pi_{\lambda_{k}}}_{\nu}}(P^{\pi_{\lambda_{k}}})^{J}(|Q^{\pi_{\lambda_{k}}}-Q_{0}|) \label{proof_15} 
\end{align}

For a fixed $j$ consider the term $\mathbb{E}_{(s,a) \sim d^{\pi_{\lambda_{k}}}_{\nu}}((P^{\pi_{\lambda_{k}}})^{J-j-1}|\epsilon_{k,j+1}|)$. We then write

\begin{eqnarray}
    \mathbb{E}_{(s,a) \sim d^{\pi_{\lambda_{k}}}_{\nu}}((P^{\pi_{\lambda_{k}}})^{J-j-1}|\epsilon_{k,j+1}|) &\le& \Bigg|\Bigg|\frac{d{(P^{\pi_{\lambda_{k}}})^{J-j-1}}{\cdot}\mu^{'}_{k}}{d{\mu_{k}}}\Bigg|\Bigg|_{\infty} {\int}\left|{\epsilon_{k,j+1}}\right|d{\mu_{k}} \\ 
   &\le&  (\phi_{\mu^{'}_{k},\mu_{j}})  \mathbb{E}_{(s,a) \sim \zeta_{\nu}^{\pi_{\lambda_{k}}}(s,a)}(|\epsilon_{k,j+1}|)\label{proof_15_1} 
\end{eqnarray}

Here $\mu_{k}$ is the measure associated with the state action distribution given by sampling from  $\zeta_{\nu}^{\pi_{\lambda_{k}}}(s,a)$ and $(P^{\pi_{\lambda_{k}}})^{J-j-1}{\cdot}\mu^{'}_{k}$, is the measure associated with applying the operator $P^{\pi_{\lambda_{k}}}$, $J-j-1$ times after sampling from the distribution $(s,a) \sim d^{\pi_{\lambda_{k}}}_{\nu}$. $(\phi_{\mu^{'}_{k},\mu_{j}})$ is the supremum of the Radon Nikodym derivative of $(P^{\pi_{\lambda_{k}}})^{J-j-1}{\cdot}\mu^{'}_{k}$ with respect to  $\zeta_{\nu}^{\pi_{\lambda_{k}}}(s,a)$. Note that in works such as \citet{fu2020single}, the Radon Nikodym derivative of with respect to the optimal policy are assumed bounded. In our case we are assuming an upper bound on a smaller quantity since $(P^{\pi_{\lambda_{k}}})^{J-j-1}{\cdot}\mu^{'}_{k}$ is closer to it's corresponding stationary distribution than the optimal policy since we have assumed ergodicity. Thus Equation \eqref{proof_15} becomes

\begin{align}
    \mathbb{E}_{(s,a) \sim d^{\pi_{\lambda_{k}}}_{\nu}}{|Q^{\pi_{\lambda_{k}}}-Q_{k,J}|}  
    \leq & \sum_{k=0}^{J-2} \gamma^{J-j-1} (\phi_{\mu^{'}_{k},\mu_{j}})  \mathbb{E}_{(s,a) \sim \zeta_{\nu}^{\pi_{\lambda_{k}}}(s,a)}(|\epsilon_{k,j+1}|)|  +   \gamma^{J}Q_{max} \label{proof_15_2} 
\end{align}

We get the second term on the right hand side by noting that from lemma F.4 of \cite{fu2020single} combined with the bounded state action space we have that $Q_{0}$ is bounded and $Q^{\pi_{\lambda_{k}}} \le \frac{1}{1-\gamma}$. Now splitting $\epsilon_{k,j+1}$ as  was done in Equation \eqref{last} we obtain

\begin{align}
    \mathbb{E}_{(s,a) \sim d^{\pi_{\lambda_{k}}}_{\nu}}{|Q^{\pi_{\lambda_{k}}}-Q_{k,J}|} 
    \leq& \sum_{j=0}^{J-2} \gamma^{J-j-1} \big((\phi_{\mu^{'}_{k},\mu_{j}})\mathbb{E}{|\epsilon^{1}_{k,j+1}|} + (\phi_{\mu^{'}_{k},\mu_{j}})\mathbb{E}{|\epsilon^{2}_{k,j+1}|} 
    \nonumber
    \\
    &+(\phi_{\mu^{'}_{k},\mu_{j}})\mathbb{E}{|\epsilon^{3}_{k,j+1}|} + (\phi_{\mu^{'}_{k},\mu_{j}})\mathbb{E}{|\epsilon^{4}_{k,j+1}|}\big) +  {\tilde{\mathcal{O}}}(\gamma^{J}) \label{proof_16}
\end{align}
The expectation on the right hand side is with respect to ${(s,a) \sim \zeta_{\nu}^{\pi_{\lambda_{k}}}(s,a)}$.
Now using Lemmas \ref{lem_1}, \ref{lem_2}, \ref{lem_3}, \ref{lem_4} we have  

\begin{eqnarray}
    \mathbb{E}_{(s,a) \sim d^{\pi_{\lambda_{k}}}_{\nu}}{|Q^{\pi_{\lambda_{k}}}-Q_{k,j}|}  &\le&
   {\tilde{\mathcal{O}}}\left(\frac{1}{\sqrt{n}}\right)     
                                +  {\tilde{\mathcal{O}}}\left(L^{-\frac{1}{4}}\right)  + {\tilde{\mathcal{O}}}\left(m^{-\frac{1}{12}}D^{\frac{7}{2}}\right) + \tilde{\mathcal{O}}(\sqrt{\epsilon_{approx}}) + {\tilde{\mathcal{O}}}(\gamma^{J}) \label{proof_17}
\end{eqnarray}

Plugging  \eqref{proof_17} into  \eqref{13_22_9} we get

\begin{eqnarray}
 \delta_{\lambda_{t}}  &\le&   \left(\frac{1}{t}\right)\delta_{\lambda_{2}}  + \frac{\alpha}{t} \sum_{k=0}^{k=t-2} \left(  {\tilde{\mathcal{O}}}\left(\frac{1}{\sqrt{n}}\right)     
                                + {\tilde{\mathcal{O}}}\left(\frac{1}{{L}^{-\frac{1}{4}}}\right)  + {\tilde{\mathcal{O}}}\left(m^{-\frac{1}{12}}D^{\frac{7}{2}}\right)  +  
                                \tilde{\mathcal{O}}(\sqrt{\epsilon_{approx}}) + {\tilde{\mathcal{O}}}(\gamma^{J}) \right)  \nonumber\\
                        &+&    \frac{1}{\sqrt{n}}{d}{M_{g}}V_{max} +  \frac{L_{J}{\cdot}{\alpha^{2}}}{t}  + \epsilon^{'} \nonumber  \\ \label{18}
     &\le&   \mathcal{\tilde{O}}\left(\frac{1}{t}\right)  + {\tilde{\mathcal{O}}}\left(\frac{1}{\sqrt{n}}\right)     
                                + {\tilde{\mathcal{O}}}\left(\frac{1}{{L}^{-\frac{1}{4}}}\right)  + {\tilde{\mathcal{O}}}\left(m^{-\frac{1}{12}}D^{\frac{7}{2}}\right) + {\tilde{\mathcal{O}}}(\gamma^{J})  + 
                                \tilde{\mathcal{O}}(\sqrt{\epsilon_{approx}})  + \tilde{\mathcal{O}}(\sqrt{\epsilon_{bias}}) \nonumber 
     \\ \label{20}   
\end{eqnarray}

\end{proof}

\section{Proof of Supporting Lemmas} \label{Upper Bounding Bellman Error}

\subsection{Proof Of Lemma \ref{lem_1}} \label{proof_lem_1}

\begin{proof}

Using Assumption \ref{assump_7} and the definition of $Q_{k,j}^{1}$ for some iteration $k$ of Algorithm \ref{algo_1} we have  
\begin{equation}
\mathbb{E}_{s,a}(T^{\pi_{\lambda_{k}}}Q_{k,j-1} - Q^{1}_{k,j})^{2} \le \epsilon_{approx}  
\end{equation}

where $(s,a) \sim \zeta_{\nu}^{\pi_{\lambda_{k}}}(s,a)$.

Since $|a|^{2}=a^{2}$ we obtain

\begin{equation}
\mathbb{E}(|T^{\pi_{\lambda_{k}}}Q_{k,j-1} - Q^{1}_{k,j}|)^{2} \le \epsilon_{approx}  
\end{equation}

We have for a random variable $x$, $ Var(x)=\mathbb{E}(x^{2}) - (\mathbb{E}(x))^{2}$ hence $\mathbb{E}(x) = \sqrt{\mathbb{E}(x^{2}) -Var(x)}$, Therefore replacing $x$ with $|T^{\pi_{\lambda_{k}}}Q^{\pi_{\lambda_{k}}} - Q_{k1}|$ we get

using the definition of the variance of a random variable we get  
\begin{eqnarray}
\mathbb{E}(|T^{\pi_{\lambda_{k}}}Q_{k,j-1} - Q^{1}_{k,j}|) &=& \sqrt{\mathbb{E}(|T^{\pi_{\lambda_{k}}}Q_{k,j-1} - Q^{1}_{k,j}|)^{2} - Var(|T^{\pi_{\lambda_{k}}}Q_{k,j-1} - Q^{1}_{k,j}|)}  \nonumber\\
\\
&\le&  \sqrt{\mathbb{E}(|T^{\pi_{\lambda_{k}}}Q_{k,j-1} - Q^{1}_{k,j}|)^{2}} 
\end{eqnarray}

Therefore by definition of $Q^{1}_{k,j}$ and assumption \ref{assump_7} we have

\begin{equation}
    \mathbb{E}(T^{\pi_{\lambda_{k}}}Q_{k,j-1} - Q^{1}_{k,j}|)\le \sqrt{\epsilon_{approx}}
\end{equation}

Since $\epsilon^{1}_{k,j}=T^{\pi_{\lambda_{k}}}Q^{\pi_{\lambda_{k}}} - Q_{k1}$ we have 
\begin{equation}
    \mathbb{E}(|\epsilon^{1}_{k,j}|) \le \sqrt{\epsilon_{approx}}
\end{equation}
\end{proof}

\subsection{Proof Of Lemma \ref{lem_2}} \label{proof_lem_2}
\begin{proof}
From Lemma \ref{sup_lem_1}, we have 

\begin{equation}
    \argminA_{f_{\theta}}\mathbb{E}_{x,y}\left(f_{\theta}(x)-g(x,y)\right)^{2}=\argminA_{f_{\theta}} \left(\mathbb{E}_{x}\left(f_{\theta}(x)-\mathbb{E}(g(y^{'},x)|x)\right)^{2}\right) \label{lem_3_1}
\end{equation}   

We label $x$ to be the state action pair $(s,a)$, $y$ is the next state action pair denoted by $(s^{'},a^{'})$. The function $f_{\theta}(x)$ to be $Q_{\theta}(s,a)$ and $g(x,y)$ to be the function $r(s,a) + {\gamma}Q_{k,j-1}(s^{'},a^{'}) $

Then the loss function corresponding to the lest hand side of Equation \eqref{sup_lem_1_1} becomes
\begin{equation}
\mathbb{E}(Q_{\theta}(s,a)-(r(s,a)+{\gamma}Q_{k,j-1}(s^{'},a^{'})))^{2} \label{lem_3_2}
\end{equation}

where  $(s,a) \sim \zeta_{\nu}^{\pi_{\lambda_{k}}}(s,a)$, $s^{'} \sim P(.|(s,a)), a^{'} \sim \pi^{\lambda_{k}}(.|s^{'}) $.

Therefore by Lemma \ref{sup_lem_1}, we have that the function $Q_{\theta}(s,a)$ which minimizes Equation \eqref{lem_3_2} it will be minimizing

\begin{equation}
\mathbb{E}_{(s,a) \sim \zeta_{\nu}^{\pi_{\lambda_{k}}}(s,a)}(Q_{\theta}(s,a) -\mathbb{E}_{s^{''} \sim P(.|s,a),a^{''} \sim \pi^{\lambda_{k}}(.|s^{'})}(r(s,a)+{\gamma}Q_{k,j-1}(s^{''},a^{''})|s,a))^{2} \label{lem_3_3}
\end{equation}

But we have from Equation \eqref{ps_3}  that

\begin{equation}
\mathbb{E}_{s^{''} \sim P(.|s,a), a^{''} \sim \pi^{\lambda_{k}}(.|s^{'})}(r(s,a)+{\gamma}Q_{k,j-1}(s^{''},a^{''}_{i})|s,a) = T^{\pi_{\lambda_{k}}}Q_{k,j-1}\label{lem_3_4}
\end{equation}

Combining Equation \eqref{lem_3_2} and \eqref{lem_3_4} we get

\begin{equation}
\argminA_{Q_{\theta}}\mathbb{E}(Q_{\theta}(s,a)-(r(s,a) + {\gamma}Q_{k,j-1}(s^{'},a^{'}_{i})))^{2} = \argminA_{Q_{\theta}}\mathbb{E}(Q_{\theta}(s,a)-T^{\pi_{\lambda_{k}}}Q_{k,j-1})^{2} \label{lem_3_4_1}
\end{equation}

The left hand side of Equation \eqref{lem_3_4_1} is $Q^{2}_{k,j}$ as defined in Definition \ref{def_2} and the right hand side is  $Q^{1}_{k,j}$ as defined in Definition \ref{def_1}, which gives us 

\begin{equation}
    Q^{2}_{k,j}=Q^{1}_{k,j}
\end{equation}

\end{proof}

\subsection{Proof Of Lemma \ref{lem_3}}\label{proof_lem_3}

\begin{proof}
 
We define $R_{X_{k},Q_{k,j-1}}({\theta})$ as

\begin{equation}
    R_{X_{k},Q_{k,j-1}}({\theta}) = \frac{1}{n} \sum^{n}_{i=1}\Bigg( Q_{\theta}(s_{i},a_{i}) - \Bigg(r(s_{i},a_{i})\nonumber\\ 
        + {\gamma}Q_{k,j-1}(s^{'}_{i},a^{'}_{i}) \Bigg)\Bigg)^{2},
\end{equation}

Here, $X_{k}$ denotes the set of tuples $(s,a,s^{'},a^{'})$ sampled at iteration $k$ of algorithm \ref{algo_1}. They are sampled from a Markov chain whose stationary state action distribution is, $(s,a) \sim \zeta_{\nu}^{\pi_{\lambda_{k}}}$. $Q_{\theta}$ is is the neural network corresponding to the parameter $\theta$ and $Q_{k,j-1}$ is the estimate of the $Q$ function obtained at iteration $k$ of the outer for loop and iteration $j-1$ of the first inner for loop of Algorithm \ref{algo_1}.

We also define the term

\begin{equation}
   L_{Q_{k,j-1}}({\theta}) = \mathbb{E}(Q_{\theta}(s,a)-(r(s,a)+{\gamma}Q_{k,j-1}(s',a'))^{2}  
\end{equation}

where ${(s,a) \sim \zeta^{\nu}_{\pi_{\lambda_{k}}}, s^{'} \sim P(.|(s,a)), a^{'} \sim \pi_{\lambda_{k}}(.|s^{'})} \nonumber\\$

We denote by $\theta^{2}_{k,j}, \theta^{3}_{k,j}$ the parameters of the neural networks $Q^{2}_{k,j}, Q^{3}_{k,j}$ respectively for notational convenience. $Q^{2}_{k,j}, Q^{3}_{k,j}$ are defined in Definition \ref{def_2} and  \ref{def_3} respectively.  

We then obtain,

\begin{eqnarray}
    R_{X_{k},Q_{k,j-1}}(\theta^{2}_{k,j}) - R_{X_{k},Q_{k,j-1}}(\theta^{3}_{k,j}) &\le&  R_{X_{k},Q_{k,j-1}}(\theta^{2}_{k,j}) - R_{X_{k},Q_{k,j-1}} 
    (\theta^{3}_{k,j}) \nonumber\\  
                                                        &&     + L_{Q_{k,j-1}}(\theta^{3}_{k,j}) - L_{Q_{k,j-1}}(\theta^{2}_{k,j}) \nonumber\\  
                                                        &&    \label{2_2_2}\\
                                                        &=&   {R_{X_{k},Q_{k,j-1}}(\theta^{2}_{k,j})- L_{Q_{k,j-1}}(\theta^{2}_{k,j})} \nonumber\\ 
                                                        &&  - {R_{X_{k},Q_{k,j-1}}(\theta^{3}_{k,j}) + L_{Q_{k,j-1}}(\theta^{2}_{k,j})} \nonumber\\
                                                        &&\label{2_2_3}\\
                                                        &\le&   \underbrace{|R_{X_{k},Q_{k,j-1}}(\theta^{2}_{k,j})- L_{Q_{k,j-1}}(\theta^{2}_{k,j})|}_{I}  \nonumber\\
                                                        &&  + \underbrace{|R_{X_{k},Q_{k,j-1}}(\theta^{3}_{k,j})- L_{Q_{k,j-1}}(\theta^{3}_{k,j})|}_{II} \nonumber\\
                                                        && \label{2_2_4}
\end{eqnarray}

We get the inequality in Equation \eqref{2_2_2} because $L_{Q_{k,j-1}}(\theta^{3}_{k,j}) - L_{Q_{k,j-1}}(\theta^{2}_{k,j}) > 0$ as  $Q^{2}_{k,j}$ is the minimizer of the loss function $ L_{Q_{k,j-1}}(Q_{\theta})$. We take the absolute value on both sides of \eqref{2_2_4}. We can do this and the sign will remain the same as the left hand side is positive since $R_{X_{k},Q_{k,j-1}}(\theta^{2}_{k,j}) \ge R_{X_{k},Q_{k,j-1}}(\theta^{3}_{k,j})$ since $\theta_{k,j}^{3}$ is the minimizer of $R_{X_{k},Q_{k,j-1}}$. Thus we get

\begin{eqnarray}
    |R_{X_{k},Q_{k,j-1}}(\theta^{2}_{k,j}) - R_{X_{k},Q_{k,j-1}}(\theta^{3}_{k,j})| &\le&   \underbrace{|R_{X_{k},Q_{k,j-1}}(\theta^{2}_{k,j})- L_{Q_{k,j-1}}(\theta^{2}_{k,j})|}_{I}  \nonumber\\
                                                        &&  + \underbrace{|R_{X_{k},Q_{k,j-1}}(\theta^{3}_{k,j})- L_{Q_{k,j-1}}(\theta^{3}_{k,j})|}_{II} \nonumber\\
                                                        && \label{2_2_4_0}
\end{eqnarray}

Consider Lemma \ref{sup_lem_0}. The loss function  $R_{X_{k},Q_{k,j-1}}(\theta)$ can be written as the mean of loss functions of the form $l(a_{\theta}(s_{i},a_{i},s^{'}_{i},a^{'}_{i}),y_{i})$ where $l$ is the square function. $a_{\theta}(s_{i}, a_{i},s^{'}_{i},a^{'}_{i})=Q_{\theta}(s_{i},a_{i})$  and $y_{i}=\Big(r(s_{i},a_{i}) + {\gamma}Q_{k,j-1}(s^{'}_{i},a^{'}_{i})\Big)$. Thus we have

\begin{eqnarray}
   & \mathbb{E}\sup_{\theta \in \Theta^{''}}|R_{X_{k},Q_{k,j-1}}(\theta)- L_{Q_{k,j-1}}(\theta)| \le  \label{2_2_4_1}\\
  & 2{\eta}^{'}\mathbb{E} \left(Rad(\mathcal{A} \circ \{(s_{1},a_{1},s^{'}_{1},a^{'}_{1}),(s_{2},a_{2},s^{'}_{2},a^{'}_{2}),\cdots,(s_{n},a_{n},s^{'}_{n},a^{'}_{n})\})\right)\nonumber  
\end{eqnarray}

Note that the expectation is over all $(s_{i},a_{i},s^{'}_{i},a^{'}_{i})$ and  the parameter set is $\Theta^{''} = \{{\theta^{2}_{k,j}}, {\theta^{3}_{k,j}}\}$. We use this set because we only need this inequality to hold for ${Q^{2}_{k,j}}$ and  ${Q^{3}_{k,j}}$ . Here  $n = |X_{k}|$, $(\mathcal{A} \circ \{(s_{1},a_{1},s^{'}_{1},a^{'}_{1}),(s_{2},a_{2},s^{'}_{2},a^{'}_{2}),\cdots,(s_{n},a_{n},s^{'}_{n},a^{'}_{n}):\theta \in \Theta^{''}\} = \{Q_{\theta}(s_{1},a_{1}), Q_{\theta}(s_{2},a_{2}), \cdots, Q_{\theta}(s_{n},a_{n}):\theta \in \Theta^{''}\}$  and $\eta^{'}_{i}$ is the Lipschitz constant for the square function  over the state action space.

Now as is shown in Proposition 11 of \citet{article}  we have that 

\begin{equation}
    \left(Rad(\mathcal{A} \circ \{(s_{1},a_{1},s^{'}_{1},a^{'}_{1}),(s_{2},a_{2},s^{'}_{2},a^{'}_{2}),\cdots,(s_{n},a_{n},s^{'}_{n},a^{'}_{n})\})\right) \le  \tilde{\mathcal{O}}\left(\frac{1}{\sqrt{n}}\right)
\end{equation}

Note that in \citet{article} the term of the form $\sum_{i=1}^{n}(f(x_{i})-\mathbb{E}f(x))$ has been upper bounded by a factor of $\tilde{\mathcal{O}}(\sqrt{n})$. The way we have defined $R_{X_{k},Q_{k,j-1}}$ means we have a term of the form $\frac{1}{n}\sum_{i=1}^{n}(f(x_{i})-\mathbb{E}f(x))$ on the left hand side of Equation \eqref{2_2_4_1} which implies that it is upper bounded by a term of the form $\tilde{\mathcal{O}}\left(\frac{1}{\sqrt{n}}\right)$. 

We use this result as the state action pairs are drawn not from the stationary state of the policy $\pi_{\lambda_{k}}$ but from a Markov chain with the same steady state distribution.
Thus we have for $\theta={\theta}^{2}_{k,j}$
\begin{eqnarray}
   \mathbb{E}|(R_{X_{k},Q_{k,j-1}}(\theta^{2}_{k,j})) - L_{Q_{k,j-1}}(\theta^{2}_{k,j})| \le  \tilde{\mathcal{O}}\left(\frac{1}{\sqrt{n}}\right) \label{2_2_5}
\end{eqnarray}
The same argument can be applied for $\theta={\theta}^{3}_{k,j}$ to get
\begin{eqnarray}
   \mathbb{E}|(R_{X_{k},Q_{k,j-1}}(\theta^{3}_{k,j})) - L_{Q_{k,j-1}}(\theta^{3}_{k,j})| \le  \tilde{\mathcal{O}}\left(\frac{1}{\sqrt{n}}\right) \label{2_2_5_1}
\end{eqnarray}
Then, plugging Equation \eqref{2_2_5},\eqref{2_2_5_1} into  Equation \eqref{2_2_4_0} we have 
\begin{equation}
\mathbb{E}\left|R_{X_{k},Q_{k,j-1}}(\theta^{2}_{k,j}) - R_{X_{k},Q_{k,j-1}}(\theta^{3}_{k,j})\right| \le \tilde{\mathcal{O}}\left(\frac{1}{\sqrt{n}}\right)  \label{2_2_5_2}
\end{equation}
Plugging in the definition of $R_{X_{k},Q_{k,j-1}}(\theta^{2}_{k,j}), R_{X_{k},Q_{k,j-1}}(\theta^{3}_{k,j})$ in equation  \eqref{2_2_5_1}, \eqref{2_2_5_2} into \eqref{2_2_4_0}  and denoting $ \tilde{\mathcal{O}}\left(\frac{1}{\sqrt{n}}\right)$ as  $\epsilon$ we get

\begin{align}
     \frac{1}{n} \sum_{i=1}^{n} \Big(\mathbb{E}|(Q^{2}_{k,j}(s_{i},a_{i})-(r(s_{i},a_{i}) + {\gamma}Q_{k,j-1}(s^{'}_{i},a^{'}_{i})))^{2} 
    - (Q^{3}_{k,j}(s_{i},a_{i})-(r(s_{i},a_{i}) + {\gamma}Q_{k,j-1}(s^{'}_{i},a^{'}_{i})))^{2}|\Big) \le \epsilon\label{2_2_5_3_1}
\end{align}

Now for a fixed $i$ consider the term $\alpha_{i}$ defined as.

\begin{eqnarray}
   \mathbb{E}|(Q^{2}_{k,j}(s_{i},a_{i})-(r(s_{i},a_{i}) + {\gamma}Q_{k,j-1}(s^{'}_{i},a^{'}_{i})))^{2} 
    -  (Q^{3}_{k,j}(s_{i},a_{i})-(r(s_{i},a_{i}) + {\gamma}Q_{k,j-1}(s^{'}_{i},a^{'}_{i})))^{2}|    \label{2_2_5_3_2}
\end{eqnarray}

where $s_{i},a_{i},s^{'}_{i},a^{'}_{i}$ are drawn from the state action distribution  at the $i^{th}$ step of the Markov chain induced by following the policy $\pi_{\lambda_{k}}$.

Now for a fixed $i$ consider the term $\beta_{i}$ defined as.

\begin{eqnarray}
    \mathbb{E}|(Q^{2}_{k,j}(s_{i},a_{i})-(r(s_{i},a_{i}) + {\gamma}Q_{k,j-1}(s^{'}_{i},a^{'}_{i})))^{2}
    -  (Q^{3}_{k,j}(s_{i},a_{i})-(r(s_{i},a_{i}) + {\gamma}Q_{k,j-1}(s^{'}_{i},a^{'}_{i})))^{2}|    \label{2_2_5_3_3}
\end{eqnarray}

where $s_{i},a_{i}$ for all $i$ are drawn from the steady state action distribution  with $(s_{i},a_{i}) \sim \zeta^{\nu}_{\pi_{\lambda_{k}}}$, $s^{'}_{i} \sim P(.|s,a)$ and $a^{'}_{i} \sim \pi^{\lambda_{k}}(.|s_{i}^{'}) $. Note here that $\alpha_{i}$ and $\beta_{i}$ are the same function with only the state action pairs being drawn from different distributions.

Using these definitions we obtain

\begin{eqnarray}
  |\mathbb{E}(\alpha_{i}) -  \mathbb{E}(\beta_{i})|  &\le& \sup_{(s_{i},a_{i},s_{i}^{'},a_{i}^{'})}|2.\max(\alpha_{i},\beta_{i})|({\kappa}_{i})   \label{2_2_5_3_4}\\
   &\le& R_{max}.p{\rho}^{i} \label{2_2_5_3_5}
\end{eqnarray}

We obtain Equation \eqref{2_2_5_3_4} by using the identity $|{\int}fd{\mu} - {\int}fd{\nu}|  \le  |\max(f)|{\int}|(d{\mu}-d{\nu})| \le  |\max(f)|{\cdot}{d_{TV}}(\mu,\nu)$, where $\mu$ and $\nu$ are two $\sigma$ finite state action probability measures and $f$ is a bounded measurable function.  We have used  $\kappa_{i}$ to represent the total variation distance between the measures of $(s,a,s^{'},a^{'})$ induced by sampling form the steady state action distribution denoted by $(s,a) \sim \zeta^{\nu}_{\pi_{\lambda_{k}}}$ and the measures of $(s_{i},a_{i},s_{i}^{'},a_{i}^{'})$  induced at the $i^{th}$ step of the Markov chain induced by following the policy $\pi^{\lambda_{k}}$. We obtain Equation \eqref{2_2_5_3_5} from  Equation  \eqref{2_2_5_3_4} by using  Assumption \ref{assump_4} and the fact from lemma F.4 of \cite{fu2020single} combined with the bounded state action space we have that $\sup_{(s_{i},a_{i})}|2.\max(\alpha_{i},\beta_{i})|$ will be bounded.

From equation \eqref{2_2_5_3_5} we get 

\begin{eqnarray}
   \mathbb{E}(\beta_{i})  &\le& \mathbb{E}(\alpha_{i}) + R_{max}.p{\rho}^{i} \label{2_2_5_3_6}
\end{eqnarray}

We get Equation \eqref{2_2_5_3_6} from Equation \eqref{2_2_5_3_5} using the fact that $|a-b| \le c$ implies that  $ \left(-c \ge   (b-a) \le c \right)$ which in turn implies $ b \le a + c $.

Using Equation \eqref{2_2_5_3_6} in equation \eqref{2_2_5_3_1} we get
\begin{eqnarray}
     &&\frac{1}{n} \sum_{i=1}^{n} \Big(\mathbb{E}|(Q^{2}_{k,j}(s_{i},a_{i})-(r(s_{i},a_{i}) + {\gamma}Q_{k,j-1}(s^{'}_{i},a^{'}_{i})))^{2}
   - (Q^{3}_{k,j}(s_{i},a_{i})-(r(s_{i},a_{i}) + {\gamma}Q_{k,j-1}(s^{'}_{i},a^{'}_{i})))^{2}| \Big) \nonumber\\
   &\le& \epsilon  + \frac{1}{n} \sum_{i=1}^{n}R_{max}.p{\rho}^{i} \nonumber\\
    &\le& \epsilon  + \frac{R_{max}}{n}p\frac{1}{1-\rho} \nonumber\\
    \label{2_2_5_3_7}    
\end{eqnarray}

In Equation \eqref{2_2_5_3_7}  $(s_{i},a_{i}) \sim \zeta^{\nu}_{\pi_{\lambda_{k}}}, s^{'} \sim P(.|s,a), a \sim \pi^{\lambda_{k}}(.|s^{'})$ for all $i$.

Since now all terms in the summation on the left hand side of Equation \eqref{2_2_5_3_7} are i.i.d,  Equation \eqref{2_2_5_3_7} is equivalent to,

\begin{eqnarray}
\mathbb{E}|(Q^{2}_{k,j}(s,a)-(r(s,a) + {\gamma}Q_{k,j-1}(s^{'},a^{'})))^{2}
   - (Q^{3}_{k,j}(s_{i},a_{i})-(r(s,a) + {\gamma}Q_{k,j-1}(s^{'},a^{'})))^{2}| &\le& \epsilon \nonumber\\
    && \label{2_2_5_4}
\end{eqnarray}

Now we apply an argument similar to the \textit{state regularity assumption} of \cite{tian2023convergence} wherein it is assumed that there exists a positive constant $\lambda^{'}$ such that $|Q_{\theta_{1}} - Q_{\theta_{2}}| \ge {\lambda}^{'}|\theta_{1} -\theta_{2}| $ for any $\theta_{1},\theta_{2} \in \Theta^{'}$.

Using the same argument for the function $F_{\theta^{'}}(\theta) = (Q_{\theta}(s,a)-(r(s,a) + {\gamma}Q_{\theta^{'}}(s^{'},a^{'})))^{2}$,  we assume there exists a positive constant $\lambda_{\theta^{'}}$ such that $|F_{\theta^{'}}(\theta_{1}) - F_{\theta^{'}}(\theta_{2})| \ge  {\lambda}_{\theta^{'}}|\theta_{1} -\theta_{2}|  $ for any $\theta_{1},\theta_{2} \in \Theta^{'}$. We then define $\lambda^{''} = \inf_{\theta^{'} \in \Theta^{'}}{\lambda}_{\theta^{'}}$.

Further we use the lipschitz property of neural networks to obtain that $|Q_{\theta_{1}}-Q_{\theta_{2}}| \le L_{Q}|\theta_{1}-\theta_{2}|$ for all $(s,a) \in \mathcal{S}{\times}\mathcal{A}$, where $L_{Q}$ is the lipschitz parameter of the neural networks where $\theta \in \Theta^{'}$. Combining these two results we obtain that 
$L_{Q}.\lambda^{''}|Q_{\theta_{1}}-Q_{\theta_{2}}| \le  |F_{\theta^{'}}(\theta_{1}) - F_{\theta^{'}}(\theta_{2})|$ for any $\theta^{'} \in \Theta^{'}$.

Applying this result to Equation \eqref{2_2_5_4} we get 

\begin{eqnarray}
{\lambda}^{''}.{L_{Q}}\mathbb{E}|Q^{2}_{k,j} - Q^{3}_{k,j}| &\le& \epsilon \nonumber\\
    && \label{2_2_5_4_1}
\end{eqnarray}

where now the expectation is only over $(s,a) \sim \zeta^{\nu}_{\pi_{\lambda_{k}}}$ since the left hand side is no longer a function of $s^{'},a^{'}$. This is equivalent to

\begin{eqnarray}
\mathbb{E}_{(s,a) \sim \zeta^{\nu}_{\pi_{\lambda_{k}}}}|Q^{2}_{k,j} - Q^{3}_{k,j}| &\le& \tilde{\mathcal{O}}\left(\frac{1}{\sqrt{n}}\right) \label{2_2_5_4_2}
\end{eqnarray}

which is the required result

\end{proof}

\subsection{Proof Of Lemma \ref{lem_4}} \label{proof_lem_4}
\begin{proof}

Consider the loss function in Definition \ref{def_3} given by 

\begin{eqnarray}
   \frac{1}{n} \sum_{i=1}^{n}\Big( Q_{\theta}(s_{i},a_{i}) - \big(r(s,a) + {\gamma}Q_{k,j-1}(s_{i}^{'},a_{i}^{'}) \big)\Big)^{2}  \label{lem_7_0} \nonumber\\
\end{eqnarray}

We can consider the function in Equation \eqref{lem_7_0} as an expectation of the loss function $\Big( Q_{\theta}(s,a) - \big(r(s,a) + {\gamma}Q_{k,j-1}(s^{'},a^{'}) \big)\Big)^{2}$ where there is an equal probability $\frac{1}{n}$ on each of the observed tuples $(s_{i},a_{i},s^{'}_{i},a^{'}_{i})$ obtained at the $k^{th}$ iteration of the outer for loop of Algorithm \ref{algo_1}.

We define the local linearization of a function $Q_{\theta}$ as follows

\begin{eqnarray}
    \bar{Q}_{\theta} = Q_{\theta_{0}} + (\theta -\theta_{0})^{T} {\nabla}Q_{\theta}  \label{lem_7_1}
\end{eqnarray}

Following the  analysis in Proposition C.4 of \citet{fu2020single} we obtain from Equation G.49 that over the random initialization of $\theta_{0}$ we have with probability at least at least $1-\exp( -\Omega(m^{-\frac{2}{3}})D^{-\frac{2}{3}})$ .

\begin{eqnarray}
    \mathbb{E}(\bar{Q}_{\theta^{'}} -  \bar{Q}_{\theta^{*}})^{2} = \tilde{\mathcal{O}}({L}^{-\frac{1}{2}}) +  \tilde{\mathcal{O}}(m^{-\frac{1}{6}}D^{7})
\end{eqnarray}

Here $\theta^{'} = \frac{1}{L}\sum_{i=1}^{L}\theta_{i}$ and $\theta^{*}$ is the minimizer of the loss function in Equation \eqref{lem_7_1} and the expectation is over the finite measure where each tuples $(s_{i},a_{i})$ has a probability mass of $\frac{1}{n}$.

Now we have from Jensen's Inequality 

\begin{eqnarray}
    \mathbb{E}|\bar{Q}_{\theta^{'}} -  \bar{Q}_{\theta^{*}}| \le  \sqrt{\mathbb{E}(\bar{Q}_{\theta^{'}} -  \bar{Q}_{\theta^{*}})^{2}} = 
 \tilde{\mathcal{O}}({L}^{-\frac{1}{4}}) +  \tilde{\mathcal{O}}(m^{-\frac{1}{12}}D^{\frac{7}{2}})
\end{eqnarray}

Consider Lemma F.3 from \citep{fu2020single}, we have with probability at $1-\exp( -\Omega(m^{-\frac{2}{3}})D^{-\frac{2}{3}})$  over the random initialization of $\theta_{0}$ that
for any $\theta \in \Theta^{'}$
\begin{equation}
    |{Q}_{\theta}(s,a) -  \bar{Q}_{\theta}(s,a)| \le \tilde{\mathcal{O}}(m^{-\frac{1}{6}}D^{\frac{5}{2}})
\end{equation}

Thus we have 

\begin{eqnarray}
    \mathbb{E}|Q_{\theta^{'}} -  Q_{\theta^{*}}| &\le& \mathbb{E}|Q_{\theta^{'}} - \bar{Q}_{\theta^{'}} + \bar{Q}_{\theta^{'}} -  Q_{\theta^{*}} - \bar{Q}_{\theta^{*}}  +  \bar{Q}_{\theta^{*}}| \\
    &\le& \mathbb{E}|\bar{Q}_{\theta^{'}} -  \bar{Q}_{\theta^{*}}| +  |{Q}_{\theta^{'}}-  \bar{Q}_{\theta^{'}}|  +  |{Q}_{\theta^{\star}} -  \bar{Q}_{\theta^{\star}}| \\
    &\le&  
 \tilde{\mathcal{O}}({L}^{-\frac{1}{4}}) +  \tilde{\mathcal{O}}(m^{-\frac{1}{12}}D^{\frac{7}{2}})
\end{eqnarray}

Now from theorem 6.10 of  \citet{10.5555/26851} we have that there exists a positive function $f_{j}:\mathcal{S}{\times}\mathcal{A} \rightarrow \mathbb{R}$ such that 

\begin{eqnarray}
    \mathbb{E}^{'}|Q_{\theta^{'}} -  Q_{\theta^{*}}| = \mathbb{E}|f_{j}{\cdot}(Q_{\theta^{'}} -  Q_{\theta^{*}})|
\end{eqnarray}

Here $\mathbb{E}^{'}$ is the expectation with respect to the state action pair sampled from $(s,a) \sim \zeta^{\pi_{\lambda_{k}}}_{\nu}$. Here the function $f_{j}$ is bounded function because $Q_{\theta^{'}}$ and $Q_{\theta^{*}}$ are bounded functions since there parameters are bounded as shown in lemma F.4 of \citet{fu2020single} and the fact that the state action space is bounded.

Thus we can say

\begin{eqnarray}
    \mathbb{E}_{(s,a) \sim \zeta^{\pi_{\lambda_{k}}}_{\nu}}|Q_{\theta^{'}} -  Q_{\theta^{*}}| \le ({\sup}{f_{j}})\mathbb{E}|Q_{\theta^{'}} -  Q_{\theta^{*}}| \le \tilde{\mathcal{O}}({L}^{-\frac{1}{4}}) +  \tilde{\mathcal{O}}(m^{-\frac{1}{12}}D^{\frac{7}{2}})
\end{eqnarray}

Note that $\theta^{'}$ is the parameter we obtain at the end of critic step denoted by $\theta_{k,j}$ and $\theta^{*}$ is by definition $\theta^{3}_{k,j}$. Therefore we get the required result.

\begin{eqnarray}
    \mathbb{E}_{(s,a) \sim \zeta^{\pi_{\lambda_{k}}}_{\nu}}|Q_{\theta_{k,j}^{3}} -  Q_{\theta_{k,j}}| \le \tilde{\mathcal{O}}({L}^{-\frac{1}{4}}) +  \tilde{\mathcal{O}}(m^{-\frac{1}{12}}D^{\frac{7}{2}})
\end{eqnarray}

\end{proof}

\end{document}